\documentclass{article}
\usepackage[utf8]{inputenc} 
\usepackage[T1]{fontenc}    
\usepackage{hyperref}       
\usepackage{url}            
\usepackage{booktabs}       
\usepackage{amsfonts}       
\usepackage{nicefrac}       
\usepackage{microtype}      
\usepackage[left=1in,top=1in,right=1in,bottom=1in,letterpaper]{geometry}

\usepackage{comment}
\usepackage{graphicx,subfigure}
\usepackage{amssymb,amsmath,color,amsthm,turnstile,txfonts,mathrsfs,bm}
\usepackage{xparse}

\usepackage{tikz}
\newcommand*\circled[1]{\tikz[baseline=(char.base)]{
            \node[shape=circle,draw,inner sep=.8pt] (char) {#1};}}

\theoremstyle{definition}

\newtheorem{definition}{Definition}
\newtheorem{theorem}{Theorem}

\newtheorem{lemma}[theorem]{Lemma}
\newtheorem{remark}{Remark}
\newtheorem{prop}{Proposition}

\newcommand{\RR}{\mathbb{R}}

\def\M{\mathcal{M}}

\def\N{\mathcal{N}}
\def\E{\textbf{E}}

\def\D{\textbf{D}}

\def\RR{\mathbb{R}}
\def\S{\mathcal{S}}

\def\T{\mathcal{T}}

\def\Z{\mathcal{Z}}

\def\real{\mathbb{R}}

\title{Chart Auto-Encoders for Manifold Structured Data}

\author{Stefan C. Schonscheck\thanks{Department of Mathematics, Rensselaer Polytechnic Institute, Troy, NY 12180, U.S.A. ({\tt schon@rpi.edu}). S. Schonsheck's reserach is supported in part by the Rensselaer-IBM AI Research Collaboration (http://airc.rpi.edu), part of the IBM AI Horizons Network (http://ibm.biz/AIHorizons).}, ~~
Jie Chen\thanks{
MIT-IBM Watson AI Lab, IBM Research. U.S.A. ({\tt chenjie@us.ibm.com})},~~
Rongjie Lai\thanks{Department of Mathematics, Rensselaer Polytechnic Institute, Troy, NY 12180,
         U.S.A. ({\tt lair@rpi.edu}). R.Lai's research is supported in part by an NSF CAREER Award DMS--1752934 and the Rensselaer-IBM AI Research Collaboration (http://airc.rpi.edu), part of the IBM AI Horizons Network (http://ibm.biz/AIHorizons). }
}
\date{}

\begin{document}

\maketitle

\begin{abstract}
Deep generative models have made tremendous advances in image and signal representation learning and generation. These models employ the full Euclidean space or a bounded subset as the latent space, whose flat geometry, however, is often too simplistic to meaningfully reflect the manifold structure of the data. In this work, we advocate the use of a multi-chart latent space for better data representation. Inspired by differential geometry, we propose a \textbf{Chart Auto-Encoder (CAE)} and prove a universal approximation theorem on its representation capability. We show that the training data size and the network size scale exponentially in approximation error with an exponent depending on the intrinsic dimension of the data manifold. CAE admits desirable manifold properties that auto-encoders with a flat latent space fail to obey, predominantly proximity of data. We conduct extensive experimentation with synthetic and real-life examples to demonstrate that CAE provides reconstruction with high fidelity, preserves proximity in the latent space, and generates new data remaining near the manifold. These experiments show that CAE is advantageous over existing auto-encoders and variants by preserving the topology of the data manifold as well as its geometry.
\end{abstract}

\section{Introduction}\label{sec:intro}

Auto-encoding~\cite{bourlard1988auto,hinton1994autoencoders,liou2014autoencoder} is a central tool in unsupervised representation learning. The latent space therein captures the essential information of a given data set, serving the purposes of dimension reduction, denoising, and generative modeling. Even for models that do not employ an encoder, such as generative adversarial networks~\cite{Goodfellow2014}, the generative component starts with a latent space. A common practice is to model the latent space as a low-dimensional Euclidean space $\real^m$ or a bounded subset of it (e.g., $[0,1]^m$), sometimes equipped with a prior probability distribution. Such spaces carry simple geometry and may not be adequate for representing complexly structured data.

A commonly held belief, known as the \emph{manifold hypothesis}~\cite{Belkin2003, fefferman2016testing}, states that real-life data often lies on, or at least near, some low-dimensional manifold embedded in a high-dimensional ambient space. Hence, a natural approach to representation learning is to introduce a low-dimensional latent space to which the data is mapped. It is desirable that such a mapping possesses basic properties such as invertibility and continuity, i.e., \emph{homeomorphism}. Challengingly, it is known that even for simple manifolds, there does not always exist a homeomorphic mapping to the Euclidean space whose dimension is the intrinsic dimension of the data. 

For example, consider a double torus shown in Figure~\ref{fig:Eight}. When one uses a plain auto-encoder to map uniform points on this manifold to $\RR^2$, the distribution of the points is distorted and the shape destroyed; whereas if one maps to $\RR^3$, some of the points depart from the mass and become outliers. Generalization suffers, too. In Figure \ref{fig:Eight} and in the supplementary material, we show the results of several variational auto-encoders with increasing complexity. They fail to generate data to cover the whole manifold; worse, the newly sampled data do not all stay on the manifold. Nearby points on the manifold may be far apart in the latent space because homeomorphism is violated.
\begin{figure}[ht]
  \centering
  \includegraphics[width=.6\linewidth]{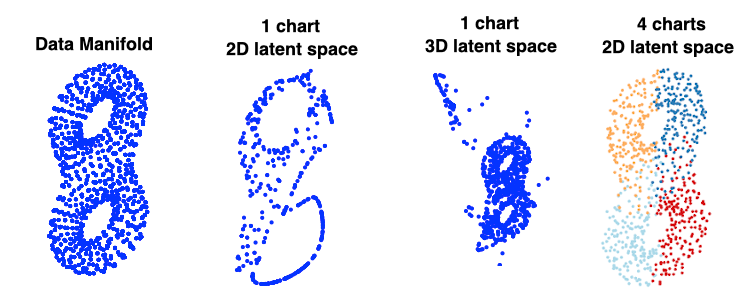}
  \caption{Left: Data on a double torus. Middle two: Data auto-encoded to a flat latent space. Right: Data auto-encoded to a 4-chart latent space.}
  \label{fig:Eight}
\end{figure}

To address this issue rigorously, we introduce a concept called \emph{$\epsilon$-faithful representation}, which quantitatively measures the topological and geometric approximation of auto-encoders to the data manifold. We show that a plain auto-encoder can not $\epsilon$-faithfully represent a manifold with non-contractible topology. To circumvent the drawbacks of existing auto-encoders, in this work, we follow the definition of manifolds in differential geometry and propose a \textbf{Chart Auto-Encoder (CAE)} to learn a low-dimensional representation of the data. Rather than using a single function mapping, the manifold is parameterized by a collection of \emph{overlapping charts}, each of which describes a local neighborhood collectively covering the entire manifold. The reparameterization of an overlapping region shared by different charts is described by the associated \emph{transition function}. Theoretically, we study a universal manifold approximation theorem using CAE with training data size and network size estimations.

As an illustration, we show to the right of Figure~\ref{fig:Eight} the same double torus aforementioned, now parameterized by using four color-coded charts. This example exhibits several characteristics and benefits of the proposed work: (i) the charts collectively cover the manifold and faithfully preserve the topology (two holes); (ii) the charts overlap (as evident from the coloring); (iii) new points sampled from the latent space remain on the manifold; (iv) nearby points on the manifold remain close in the latent space; and (v) because of the preservation of geometry, one may accurately estimate geometric proprieties (such the geodesics).

In this paper, we make the distinction between \emph{geometric} and \emph{topological} error in generative models: geometric errors occur when the model produces results far from the training data manifold, whereas topological errors occur when the model fails to capture portions of the manifold
or causing nearby points far apart in the latent space.

\subsection{Related Work}\label{sec:related}

Exploring the low-dimensional structure of manifolds has led to many dimension reduction techniques in the past two decades~\cite{tenenbaum2000global, roweis2000nonlinear, cox2001, Belkin2003, He2003, zhang2004principal, Kokiopoulou2007, maaten2008visualizing}. Many of these classical methods employ a flat Euclidean space for embedding and may lose information as aforementioned. Auto-encoders use a decoder to serve as the inverse of a dimension reduction. The latent space therein is still flat. One common approach to enhancing the capability of auto-encoders is to impose a prior distribution on the latent space (e.g., VAE~\cite{kingma2013auto}). The distributional assumption (e.g., Gaussian) introduces low-density regions that sometimes depart from the manifold. Then, paths in these regions either trace off the manifold or become invariant.

The work~\cite{falorsi2018explorations} introduced a non-Euclidean latent space to guarantee the existence of a homeomorphic representation, realized by a homeomorphic variational auto-encoder. Requiring the knowledge of the topological class of the data set and estimation of a Lie group action on the latent space limit this approach to be applied in practice. 
If the topology of the data is relatively simple (e.g., a sphere or torus), the computation is amenable; but for more complexly structured data sets, it is rather challenging. Similarly, several recent work~\cite{davidson2018hyperspherical, rey2019disentanglement, connor2019representing} studies auto-encoders with (hyper-)spherical latent spaces. These methods allow for the detection of cyclical features (prior information is needed) but offer little insight into the topology of a general manifold. Our method is completely unsupervised.

Under the manifold hypothesis, most provable neural network approximation results~\cite{csaji2001approximation, gyorfi2006distribution, shaham2018provable, chen2019efficient} for low dimensional manifolds focus on approximating a function $f$ supported on or near some smooth $d$-dimensional manifold isometrically embedded in $\RR^m$ with $d \ll m$. This setup models tasks such as recognition, classification or segmentation of data. However, when analyzing generative models such as auto-encoders and GANs, it is more interesting to study how well a model can actually represent the manifold $\M$ given some training data $\{x\}_{i=1}^n$ sampled from $\M$. Recently, \cite{khrulkov2019universality} showed an existence proof of universal approximation for data distributed on compact manifolds without providing any bounds for the size of the networks or their approximation quality. To the best of our knowledge, our work is for the first time to theoretically analyze topology and geometry approximation behavior of neural networks to a data manifold with bounds on the training data size and network size. Additionally, our work also introduces an implementable neural network architecture and demonstrates its effectiveness in real data. 
\section{Background on Manifold and Chart Based Parameterization}\label{sec:Manifolds}
\vspace{-.04in}
\begin{figure}[ht]
  \centering
  \begin{minipage}{0.3\linewidth}
    \centering
    \includegraphics[width=1\linewidth]{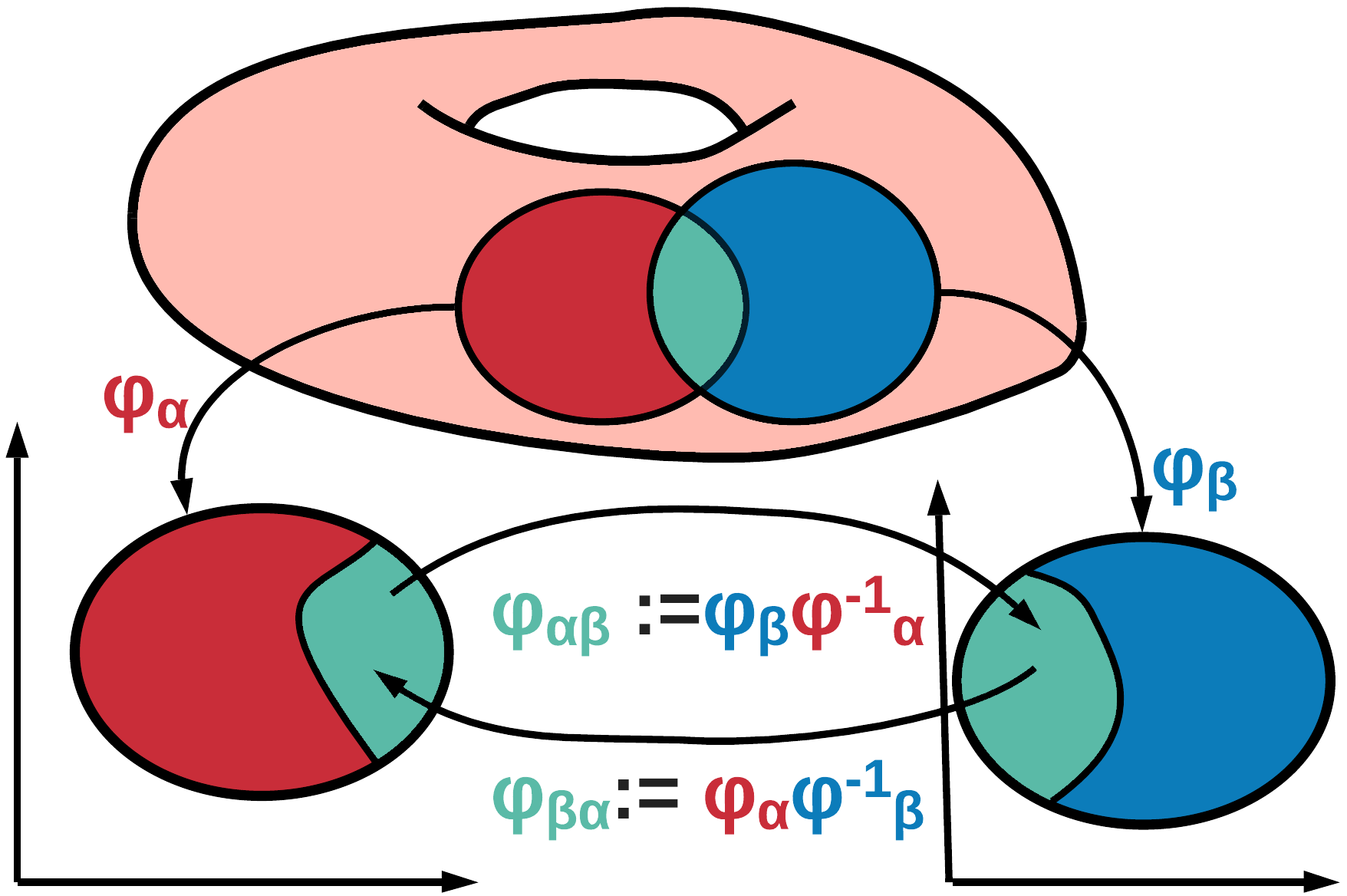}
  \end{minipage}\hfill
  \begin{minipage}{0.69\linewidth}
    \centering
    \includegraphics[width=1\linewidth]{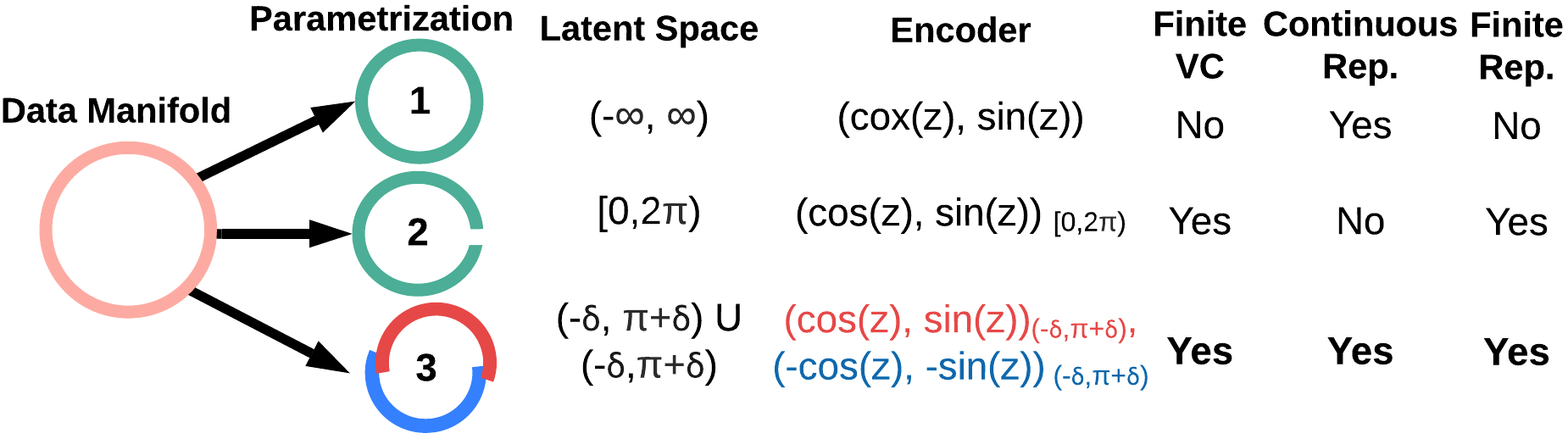}
  \end{minipage}
  \caption{Left: Illustration of a manifold. Right: Possible parameterizations of a circle. The manifold approach (bottom) preserves all desired properties.}
  \label{fig:Manifold}
\end{figure}
A manifold is a topological space locally homeomorphic to a Euclidean domain. More formally, a $d$-dimensional manifold is defined as a collection of pairs $\{(\M_\alpha,\phi_\alpha)\}_{\alpha}$, referred to as \emph{charts}, where $\{\M_\alpha\}_\alpha$ are open sets satisfying $\M= \bigcup_\alpha \M_\alpha$. Each $\M_\alpha$ is homeoporphic to an open set $\Z_\alpha\subset \real^d$ through the coordinate map $\phi_\alpha:\M_\alpha\rightarrow \Z_\alpha$. Different charts can be glued together through \emph{transition functions} $\phi_{\alpha\beta}:\phi_\alpha(\M_\alpha \cap \M_\beta) \rightarrow \phi_\beta(\M_\alpha \cap \M_\beta)$ satisfying cyclic conditions (see Figure~\ref{fig:Manifold} left). Smoothness of the transition functions controls the smoothness of the manifold. A well-known result from differential geometry states that any compact manifold can be covered by a finite number of charts which obey these conditions. 
See \cite{lee2013smooth} for a thorough review.

To illustrate the necessity of a multi-chart parameterization, we consider a simple example: finding a latent representation of data sampled from the 1-dimensional circle $S^1$ embedded in $\RR^2$ (Figure~\ref{fig:Manifold} right). A simple (non-chart) parameterization is $(\cos(z),\sin(z))$, with $z\in (-\infty,\infty)$. However, approximating this parameterization with a finite neural network is impossible, since $z$ is unbounded and hence any multi-layer perceptron will have an infinite Vapnik-Chervonenkis dimension~\cite{blumer1989learnability}. One obvious alternative is to limit $z \in [0,2\pi)$, but this parameterization introduces a discontinuity and breaks the topology (nearby points straddling across the discontinuity have distant $z$ values).
Following the definition of manifolds, we instead parameterize the circle using two charts, with $\delta$ denoting chart overlap:
\begin{equation}
  \begin{split}
    \phi_{\alpha}:  ( 0 -\delta ,\pi + \delta)  \rightarrow S^1 ,  & \quad z_\alpha \mapsto (\cos(z),\sin(z)), \\
    \phi_{\beta}:  (0 - \delta ,\pi + \delta) \rightarrow S^1,  & \quad z_\beta\mapsto  (-\cos(z),-\sin(z)), \\
    \phi_{\alpha\beta}:  (-\delta,\delta) \rightarrow (\pi-\delta, \pi+\delta),  & \quad z_\alpha \mapsto z_\alpha + \pi,\\
    \phi_{\alpha\beta}:  (\pi - \delta, \pi+ \delta) \rightarrow ( -\delta,  \delta),  & \quad z_\alpha \mapsto z_\alpha - \pi.
  \end{split}
\end{equation}
Although this function is cumbersome to write, it is more suitable for representation learning, since each encoding function can be represented with finite neural networks. Moreover, the topological and geometric information of the data is maintained.

For more complex manifolds, we use an auto-encoder to parameterize the coordinate map and its inverse. Specifically, given a manifold $\M$, typically embedded in a high dimensional ambient space $\RR^m$, the encoding network constructs a local parameterization $\phi_\alpha$ from the data manifold to the latent space $\Z_\alpha$; and the decoding network maps $\Z_\alpha$ back to the data manifold $\M$ through $\phi_\alpha^{-1}$.

\section{Universal Manifold Approximation}
\label{sec:theory}

In this section, we rigorously address the topology and geometry approximation behaviors in auto-encoders. We show that topology preservation in auto-encoders is a necessary condition to approximate data manifold $\epsilon$-closely. Moreover, we study a universal manifold approximation theorem based on multi-chart parameterization and provide estimations of training data size and network size. We defer detailed proofs of all statements to the supplemental material (Section~\ref{sec:proofs}).

Mathematically, we denote an auto-encoder of $\M\subset\real^m$ by a 3-tuple $(\Z; \E,\D)$. Here, $\Z = \E(\M)$ represents the latent space; $\E$, representing the encoder, continuously maps $\M$ to $\Z$; and $\D$, representing the decoder, continuously generates a data point $\D(z)\in\real^m$ from a latent variable $z\in\Z$. A plain auto-encoder refers to an auto-encoder whose latent space is a simply connected Euclidean domain. 

\paragraph{Faithful representation} To quantitatively measure an auto-encoder, we introduce the following concepts to characterize it.

\begin{definition}[Faithful Representation]
An auto-encoder $(\Z; \E,\D)$ is called a faithful representation of $\M$ if $x = \D\circ\E(x), \forall x\in\M$. An auto-encoder is called an $\epsilon$-faithful representation of $\M$ if $\displaystyle  \sup_{x\in\M} \|x - \D\circ\E(x)\| \leq \epsilon$. 
\end{definition}
To characterize manifolds and quality of auto-encoders, we reiterate the concept \emph{reach} of a manifold~\cite{federer1959curvature}. 
Given a $d$-dimensional compact data manifold $\M\subset \real^m$, let $\displaystyle \mathcal{G} = \Big \{y\in \real^m~|~\exists p\neq q\in\M \text{~satisfying~} \|y - p\| = \|y- q\| = \inf_{x\in\M}\|x - y\| \Big\} $. The reach of $\M$ is defined as $\displaystyle \tau (\M)= \inf_{x\in\M,y\in\mathcal{G}}\|x-y\|$.

\begin{theorem}
\label{thm:faithfulrep} Let $\M$ be a $d$-dimensional compact manifold. 
If an auto-encoder $(\Z; \E,\D)$ of $\M$ is an  $\epsilon$-faithful representation with $\epsilon < \tau(\M)$, then $\Z$ and $\D(\Z)$ must be homeomorphic to $\M$. Particularly, a $d$-dimensional compact manifold with non-contractible topology can not be $\epsilon$-faithfully represented by a plain auto-encoder with a  latent space $\Z$ being a $d$-dimensional simply connected domain in $\real^d$.
\end{theorem}

This theorem provides a necessary condition of the latent space topology for a faithful representation. It implies that a data manifold with complex topology can not be $\epsilon$-faithfully represented by an auto-encoder with a simply connected latent space like $\real^d$ used in plain auto-encoders. For example, a plain auto-encoder with a single 2 dimensional latent space cannot $\epsilon$-faithfully represent a sphere. 


\subsection{Main Theoretical Results}
To address the issue of topology violation in plain auto-encoders,  we propose a multi-chart model based on the definition of manifolds. We discuss the main results for approximating data manifolds using multi-chart auto-encoders with training data size and network size estimation. 

A data set $X =\{x_i\}_{i=1}^n\subset\M$ is called $\delta$-dense in $\M$ if    $\text{dist}(X,p) = \min_{x\in X}\|x - p\| <\delta, \forall p\in\M$.  We write $B^d_r$ a $d$-dimensional radius $r$ ball with $\displaystyle vol(B^d_r)  = \pi^{d/2}r^d/\Gamma(1+ d/2)$.
\begin{theorem}[Universal Manifold Approximation Theorem] Consider a d-dimensional compact data manifold $\M\subset \mathbb{R}^m$ with reach $\tau$ and $\displaystyle C = vol(\M)/vol(B^d_1)$. Let $X=\{x\}_{i=1}^n$ be a training data set drawn uniformly randomly on $\M$. For any $0< \epsilon <\tau/2$, if the cardinality of the training set $X$ satisfies
\begin{equation}
\label{eqn:SampleNum}
n > \beta_1 \Big(\log (\beta_2) + \log (1/\nu) \Big) \approx O(-d \epsilon^{-d}\log\epsilon)
\end{equation}
where $ \displaystyle \beta_1 = C~ \Big(\frac{\epsilon}{4}\Big)^{-d}\Big(1 - (\frac{\epsilon}{8\tau})^2\Big)^{-d/2} $ and $\displaystyle \beta_2 = C~\Big(\frac{\epsilon}{8}\Big)^{-d}\Big(1 - (\frac{\epsilon}{16\tau})^2\Big)^{-d/2} $, then
 based on the training data set $X$, there exists a CAE $(\Z,\E,\D)$ with $L>d$ charts $\epsilon$-faithfully representing $\M$  with probability $1 - \nu$. 
 In other words, we have 
\[  \sup_{x\in\M} \|x - \D\circ\E(x)\| \leq \epsilon. \]
Moreover, the encoder $\E$ and the decoder $\D$  has at most $ O(Lmd \epsilon^{-d - d^2/2}(-\log^{1+d/2}\epsilon))$ parameters and $\displaystyle   O(-d^2\log_2\epsilon/2)$ layers.

\label{thm:UMA}
\end{theorem}

This main result characterizes the approximation behavior of the CAE topologically and geometrically.  From theorem \ref{thm:faithfulrep}, we have that the latent space $\Z$ in the above network is homeomorphic to the data manifold and that the generating manifold $\D(\Z)$ preserves the topology of $\M$. Moreover, this theorem  provides estimations of requiring training data size and a network size to geometrically approximate the data manifold $\epsilon$-closely. 

\paragraph{Sketch of proof}
We first apply a result from Niyogi-Smale-Weinberger~\cite{niyogi2008finding} to obtain a estimation of the cardinality of the training set $X$ on $\M$ satisfying that $X$ is $\epsilon/2$-dense on $\M$. Then, we conduct a constructive proof to show that there exists a network satisfying the required accuracy. 

We begin with dividing the manifold $\M$ into $L$ charts satisfying $\M = \bigcup_\ell \M_\ell$. We parameterize each chart $\M_\ell$ on a $d$-dimensional tangent space $\Z_\ell$ using the $\log$ map. Then, based on a simplicial structure induced from the image of the training data set on the tangent space, we obtain a simplicial complex $\S_\ell$ whose vertices are provided by the training data on $\M_\ell$. After that, we construct a neural network to represent the piecewise linear map between the latent space and $\S_\ell$. This gives an essential ingredient to construct an encoder $\E_\ell$ and a decoder $\D_\ell$. Furthermore, we also argue that the difference between $\M_\ell$ and its simplicial approximation $\S_\ell$ is bounded above by $\epsilon$. More precisely, this local chart approximation can be summarized as:

\begin{theorem}[Local chart approximation]Consider a geodesic neighborhood $\M_r(p) = \{x\in\M~|~ d(p,x) < r \}$ around $p\in\M$. For any $0 < \epsilon  < \tau(\M)$, if $X= \{x_i\}_{i=1}^n$ is an $\epsilon/2$-dense sample drawn uniformly randomly on $\M_r(p)$, then there exists an auto-encoder $(\Z,\E,\D)$ which is an $\epsilon$-faithful representation of $\M_r(p)$. In other words, we have 
\[  \sup_{x\in\M_r(p)} \|x - \D\circ\E(x)\| \leq \epsilon. \]
Moreover, $\E$ and $\D$  have at most $O(mdn^{1+d/2})$ parameters and $O(\frac{d}{2} \log_2(n))$ layers.
\label{thm:localchart}
\end{theorem}
After the construction for each local chart, the global theorem can be obtained by patching together results of each local chart construction. This leads to the desired CAE. \hfill\qedsymbol

\section{Network Architecture}\label{sec:Arch}
\begin{figure}[ht]
  \centering
  \includegraphics[width=.9\textwidth]{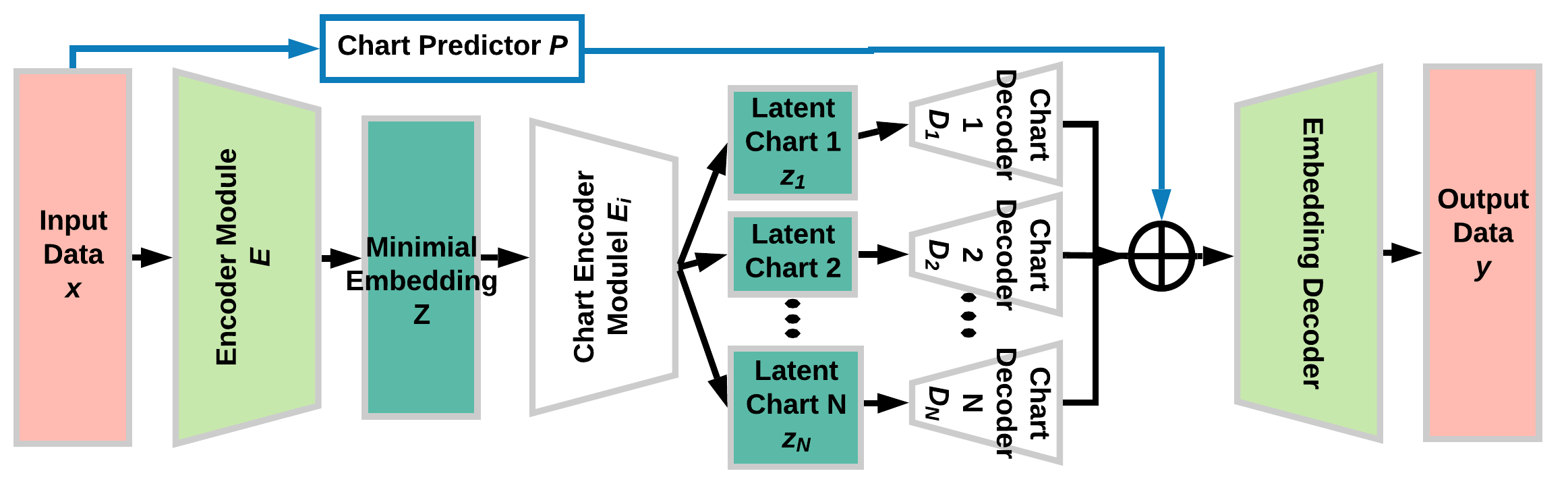}
  \vskip -0.1in
  \caption{Architecture diagram of CAE.}
  \label{fig:Archetecture}
\end{figure}
Motivated from the aforementioned analysis on the necessity of multi-charts,  we propose CAE (Figure~\ref{fig:Archetecture}) to integrate the manifold structure in the latent space. An input data point $x\in\RR^m$ is passed into an encoding module $\textbf{E}$, which creates an initial latent representation $z\in\RR^l$. Next, a collection of chart parameterizations---encoders $\textbf{E}_{\alpha}$ as analogy of $\phi_{\alpha}$---map $z$ to several chart spaces $\Z_{\alpha}$, which collectively define the multi-chart latent space. Each chart representation $z_{\alpha}\in \Z_{\alpha}$ is then passed into the corresponding decoding function---a chart decoder $\textbf{D}_{\alpha}$ as analogy of $\phi_{\alpha}^{-1}$---which produces an approximation $w_{\alpha}$ of the initial latent representation $z$. Next, a chart prediction module $\textbf{P}$ decides which chart(s) $x$ lies on and consequently selects the corresponding $w_{\alpha}$('s), which, after being mapped by the final decoder $\D$, becomes the reconstruction of $x$. The chart transition functions may be recovered by composing the chart decoders, initial encoder and the chart encoders. Hence, their explicit representations are not essential to the neural network architecture and we defer the discussion to the supplementary material.

\paragraph{Initial Encoder.}
Note that a $d$-dimensional manifold, regardless of the ambient dimension $m$, has a $C^1$ \emph{isometric} embedding in $\real^{2d}$, according to the Nash--Kuiper theorem~\cite{nash1954c1,kuiper1955c1}. Hence,
the initial encoder $\textbf{E}$ serves as a dimension reduction step, which yields a low dimensional \emph{isometric} embedding of the data from $\RR^m$ to $\RR^l$. For example, given an $\RR^3$ torus embedded in $\RR^{1000}$, the initial encoder maps from $\RR^{1000}$ to a lower dimensional space, ideally $\RR^3$. Note that however three is not the \emph{intrinsic} dimension of the torus (rather, two is); hence, a subsequent chart encoder to be discussed soon serves the purpose of mapping from $\RR^3$ to $\RR^2$.  Ideally, the initial dimension reduction step preserves the original topological and geometric information of the data manifold by reducing to the minimal isometric embedding dimension. A benefit of using an initial encoder is to reduce the subsequent computational costs in decoding. This step can be replaced with a homeomorphic variational auto-encoder~\cite{falorsi2018explorations} when the topology is known, or with an appropriately chosen random projection~\cite{baraniuk2009random,cai2018enhanced}.

\paragraph{Chart Encoder.}
This step locally parameterizes the data manifold to the chart space, whose dimension is ideally the intrinsic dimension of the manifold. The chart splits are conducted through a small collection of networks $\{\textbf{E}_\alpha\}_{\alpha}$ that take $z \in \RR^l$ as input and output several local coordinates $z_\alpha\in \Z_\alpha$. The disjoint union $\mathcal{Z} = \bigsqcup_{\alpha=1}^N \Z_\alpha$ is the multi-chart latent space. In practice, we set $\Z_\alpha = (0,1)^d$ for each $\alpha$ and regularize the Lipschitz constant of the corresponding encoding map to control the size and regularity of the region $\M_\alpha \subset \M$.

\paragraph{Chart Decoder.}
Each latent chart is equipped with a decoder $\textbf{D}_\alpha$, which maps from the chart latent space $\Z_\alpha$ back to the embedding space $\RR^l$.

\paragraph{Chart Prediction.}
The chart prediction module $\textbf{P}$ produces confidence measure $p_{\alpha}$ for the $\alpha$-th chart. For simplicity we let the $p_{\alpha}$'s be probabilities that sum to unity. Ideally, if the input point lies on a single chart, then $p_{\alpha}$ should be one for this chart and zero elsewhere. If, on the other hand, the input point lies on more than one  (say, $k$) overlapping chart, then the ideal $p_{\alpha}$ is $1/k$ for these charts. In implementation, one may use the normalized distance of the data point to the chart center as the input to $\textbf{P}$. However, for complexly structured data, the charts may have different sizes (smaller for high curvature region and larger for flat region), and hence the normalized distance is not a useful indication. Therefore, we use $x$, $z$, and/or $z_{\alpha}$ as the input to $\textbf{P}$ instead. 

\paragraph{Embedding Decoder.}
The final decoder $\D$, which is shared by all charts, works as the inverse of the initial encoder $\E$, mapping from the embedding space $\RR^l$ to the original ambient space $\RR^m$.

\paragraph{Final Output.}
If we summarize the overall pipeline, one sees that CAE produces $y_\alpha = \textbf{D}\circ\textbf{D}_\alpha\circ\textbf{E}_\alpha\circ\textbf{E}(x)$ for each chart as a reconstruction to the input $x$. Typically, the data lies on only one or at most a few of the charts, the confidence of which is signaled by $p_{\alpha}$. If only one, the corresponding $y_{\alpha}$ should be considered the final output; whereas if more than one, each of the correct $y_{\alpha}$'s should be similarly close to the input and thus taking either one is sensible. Thus, we select the $y_{\alpha}$ that maximizes $p_{\alpha}$ as the final output.

\paragraph{Loss Function.}
Recall that a chart decoder output is $y_\alpha = \textbf{D}\circ\textbf{D}_\alpha\circ\textbf{E}_\alpha\circ\textbf{E}(x)$; hence, $e_\alpha = \|x - y_\alpha \|^2$ denotes the reconstruction error for the chart indexed by $\alpha$. If $x$ lies on only one chart, this chart should be the one that minimizes $e_{\alpha}$. Even if $x$ lies on more than one chart, the minimum of $e_{\alpha}$ is still a sensible reconstruction error overall. Furthermore, to obtain sensible chart prediction probabilities $\{p_{\alpha}\}$, we will take the cross entropy between them and $\{\ell_\alpha = \mathrm{softmax}(-e_\alpha)\}$ and minimize it. If $x$ lies on several overlapping charts, on these charts the $y_\alpha$'s are similar and off these charts, the $y_\alpha$'s are bad enough that the softmax of $-e_\alpha$ is close to zero. Hence, minimizing the cross entropy ideally produces equal probabilities for the relevant charts and zero probability for the irrelevant ones.

Summarizing these two considerations, we use the loss function
\begin{equation}\label{eqn:LossPP}
  \mathcal{L}(x,W):=   \Big( \min_\alpha e_\alpha \Big) - \sum_{\beta=1}^L  \ell_{\beta} \log (p_{\beta}),
\end{equation}
where $W$ denotes all network parameters and $L$ is the number of charts. Here the first term measures the fidelity of the model and the second term trains the chart prediction module. An additional regularization term is also used in the loss to control the distortion of the encoders by penalizing their Lipschitz constant. 

\paragraph{Regularization}\label{sec:Reg}

We introduce regularization to stabilize training by balancing the size of $\M_\alpha$ and avoiding a small number of charts dominating the data manifold. For example, a sphere $S^2$ needs at least two 2-dimensional charts. However, if we regularize the network with only $l_2$ weight decay, it may be able to well reconstruct the training data by using only one chart but badly generalizes, because the manifold structure is destroyed.

The idea is to add a Lipschitz regularization to the chart encoders to penalize mapping nearby points far away. Formally, the \emph{Lipschitz constant} of a function $f$ is $\sup_{x \neq y} |f(y)-f(x)|/|x-y|$. Since the chart spaces are fixed as $(0,1)^d$, controlling the Lipschitz constant of a chart function will control the maximum volume of decoding region $\textbf{D}_{\alpha}((0,1)^d)$ on the data manifold. 

The Lipschitz constant of a composition of functions can be upper bounded by the product of those of the constituent functions. Moreover, the Lipschitz constant of a matrix is its spectral norm, and that of ReLU is 1. Hence, we can control the upper bound of the Lipschitz constant of a chart encoder function by regularizing the product of the spectral norms of the weight matrices in each layer.

To summarize, denote by $W_{\alpha}^k$ the weight matrix of the $k$th layer of $\textbf{E}_{\alpha}$. Then, we use the regularization
\begin{equation}\label{eqn:regularization}
  \mathcal{R}_{Lip}:= \left(\max_\alpha \prod_{k} ||W_\alpha^k ||_2\right) + \frac{1}{N} \sum_{\beta=1}^N  \prod_{k} ||W_\beta^k ||_2 .
\end{equation}
Here, the first term aims at stopping a single chart from dominating, and the second term promotes the smoothness of each chart.

\paragraph{Pre-Training}\label{sec:Initialization}

Since CAE jointly predicts the chart outputs and chart probabilities, it is important to properly initialize the model, so that the range of each decoder lies somewhere on the manifold and the probability that a randomly sampled point lies in each chart is approximately equal. To achieve so, we use furthest point sampling (FPS) to select $N$ points $x_{\alpha}$ from the training set as seeds for each chart. Then, we separately pre-train each chart encoder and decoder pair, such that $x_{\alpha}$ is at the center of the chart space $U_{\alpha}$. We further define the chart prediction probability as the categorical distribution and use it to pre-train the chart predictor. The loss function for each $\alpha$ is
\begin{equation}
    \label{eqn:InitLoss}
  \mathcal{L}_{init}(x_\alpha) := \|x_\alpha - \textbf{D}\circ \textbf{D}_\alpha\circ \textbf{E}_\alpha \circ \textbf{E}(x_\alpha)\|^2  
  + \|\textbf{E}_\alpha\circ \textbf{E}(x_\alpha)
  - [.5]^d\|^2 + \sum_{\beta=1}^N \delta_{\alpha \beta} \log (p_\beta).
\end{equation}

We can extend this pre-training idea to additionally ensure that the charts are oriented consistently, if desirable.  To do so, we take a small sample of points $\mathcal{N}(c_\alpha)$ around the center $c_{\alpha}$ of the $\alpha$-th chart and use principal component analysis (PCA) to define a $d$-dimensional embedding of this local neighborhood. Let the embeddings be $\hat{x}_\alpha(x) :=\frac{1}{C_\alpha} W_\alpha x + b_\alpha$ for all $x \in \mathcal{N}(c_\alpha)$, where $W_\alpha$ is the optimal orthogonal projection from $U_\alpha$ to $\RR^d$, $b_\alpha$ is used to shift $\hat{x}_\alpha(c_\alpha)$ to $[.5]^d$, and $C_i$ is chosen as a local scaling constant. Then, we can use this coordinate system to initialize the chart orientations by minimizing an additional regularization:
\begin{equation}\label{eqn:PCA}
  \mathcal{R}_{cords} =  \sum_{\alpha=1}^N\sum_{x \in \mathcal{N}(c_\alpha)} \langle \textbf{E}_\alpha \circ \textbf{E}(x) , \hat x_\alpha(x) \rangle.
\end{equation}

We remark that although the training and pre-training altogether share several similarities with clustering, the model does more than that. The obvious distinction is that CAE eventually produces overlapping charts, which are different from either hard clustering or soft clustering. One may see a deeper distinction from the training insights. The pre-training ensures that each decoder is on the manifold, so that when training begins no decoder stays inactive. However, during training the charts may move, overlap, and even disappear. The last possibility enables us to obtain the correct number of charts \emph{a posteriori}, as the next subsection elaborates.

\paragraph{Number of Charts}\label{sec:num.chart}

Since it is impossible to know \emph{a priori} the number $N$ of charts necessary to cover the data manifold, we over-specify $N$ and rely on the strong regularization~\eqref{eqn:regularization} to eliminate unnecessary charts. During training, a chart function $\textbf{E}_{\alpha}$ not utilized in the reconstruction of a point (i.e., $p_{\alpha} \approx 0$) does not get update from the loss function. Then, adding any convex penalty centered at 0 to the weights of $\textbf{E}_{\alpha}$ will result in weight decay and, if a chart decoder is never utilized, its weights will go to zero. In practice, we can remove these charts when the norm of the chart decoder weights falls below some tolerance. This mechanism offers a means to obtain the number of charts \emph{a posteriori}. We will show later a numerical example that illustrates that several charts do die off after training.

\paragraph{Chart Transition Functions}\label{sec:trans}

A key feature of the chart based parameterization in differential geometry is the construction of chart transition functions. As shown in Figure~\ref{fig:Manifold}, some points on the manifold may be parameterized by multiple charts. Let $\phi_{\alpha}$ and $\phi_{\beta}$ be two chart functions with chart overlap $M_{\alpha} \cap M_{\beta} \neq \emptyset$; then, the chart transition function $\phi_{\alpha \beta} = \phi_{\beta} \phi^{-1}_{\alpha}$.

In our model, the chart decoder $\textbf{D}\circ\textbf{D}_{\alpha}$ plays the role of $\phi_{\alpha}^{-1}$ and the composition $\textbf{E}_\beta \circ \textbf{E}$ plays the role of $\phi_{\beta}$. Hence, the chart transition function can be modeled by the composition:
\begin{equation}
  \phi_{\alpha \beta} :\Z_\alpha\cap \Z_\beta \rightarrow \Z_\beta\cap \Z_\alpha, \quad z_\alpha \mapsto  \textbf{E}_\beta \Bigg( \textbf{E} \Big( \textbf{D} \big( \textbf{D}_\alpha (z_\alpha) \big) \Big) \Bigg).
\end{equation}
Note that if $x \in \mathcal{M}_{\alpha} \cap \mathcal{M}_{\beta}$, then to obtain a high-quality transition function we need
\circled{1} $p_{\alpha}(x) \approx p_{\beta}(x)$,
\circled{2} $ x \approx \textbf{D} ( \textbf{D}_\alpha ( \textbf{E}_ \alpha( \textbf{E}(x) ))) $, and
\circled{3} $ x \approx \textbf{D} ( \textbf{D}_\beta ( \textbf{E}_\beta ( \textbf{E}(x) ))) $.
Each of these conditions are naturally met if the loss function~\eqref{eqn:LossPP} is well minimized. To gauge the accuracy of such transition functions, one may re-encode the decoded data in a second pass:
\begin{multline}
  \mathcal{R}_{cycle}(x):= 
  \|x - \textbf{D} \circ \textbf{D}_\beta \circ \textbf{E}_\beta \circ \textbf{E} \circ \textbf{D} \circ \textbf{D}_\alpha \circ \textbf{E}_\alpha \circ \textbf{E} (x) \|
  + \|x - \textbf{D} \circ \textbf{D}_\alpha \circ \textbf{E}_\alpha \circ \textbf{E} \circ \textbf{D} \circ \textbf{D}_\beta \circ \textbf{E}_\beta\circ \textbf{E} (x)\|.   
\end{multline}
The residual $\mathcal{R}_{cycle}(x)$ measures the error in chart transition and reconstruction.

All CAE components may be implemented by using fully connected and/or convolution layers (with ReLU activation). Details of the network architectures used in our experiments  are given in Section~\ref{app:Networks} in the supplemental material.

\section{Numerical Results}\label{sec:results}
In this section, we analyze the performance of CAE on both synthetic and benchmark data sets. We show a few illustrative examples here and give many more in the supplementary material (Section~\ref{app:Experiments}).

The implementation uses Tensorflow \cite{abadi2016tensorflow} and the built-in ADAM optimizer. The learning rate is set as \texttt{3e-4}, batch size is 64, number of epoch is 100, and the penalty for the Lipschitz regularization is \texttt{1e-2}. Demo code is available at \url{https://anonymous.4open.science/r/a40668ab-7542-4028-8709-694142a985da/}.

\textbf{Chart overlap.}
The major distinction between CAE and existing auto-encoders is that the latent space uses multiple overlapping charts. The overlap allows for smooth transition from one chart to another (through the formalism of transition functions). In Figure~\ref{fig:overlap} we show a synthetic manifold of a cat shape. This is a 1-dimensional manifold which has the same topology as a circle. As motivated in Section~\ref{sec:Manifolds}, a single-chart parameterization will destroy the topology. Here, we use four overlapping charts. The chart probability $p_{\alpha}$ varies smoothly as it moves away from the chart boundary to the interior. Taking the max of the probabilities for each point, one sees the reconstruction of the cat on the upper-right plot of Figure~\ref{fig:overlap}.
\begin{figure}[ht]
  \centering
  \includegraphics[width=.8\linewidth]{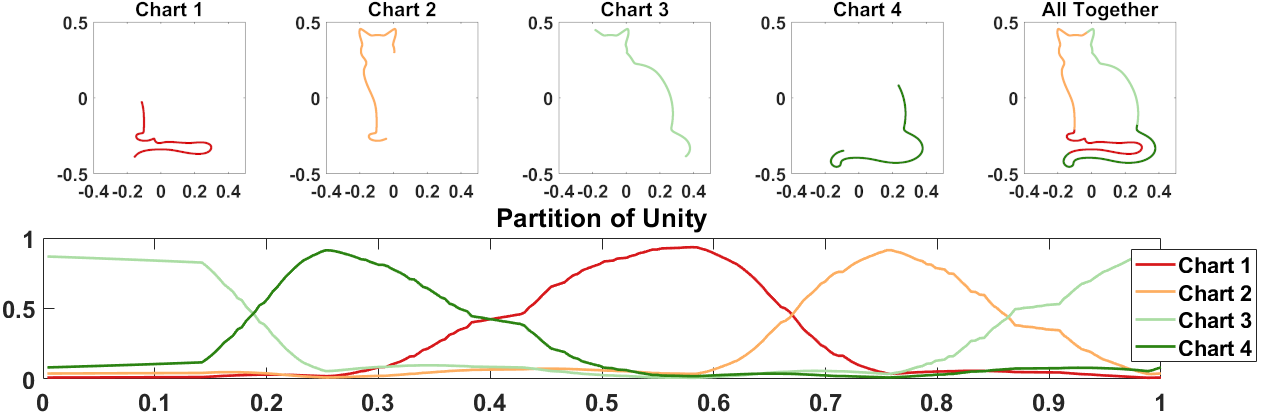}
  \caption{Top: The first four are individual charts and the last one is a concatenation of them by taking the max of chart probabilities $p_\alpha$. Bottom: Variation of $p_\alpha$ for each training point on the manifold.}
  \label{fig:overlap}
\end{figure}

\textbf{Topology preservation.}
We move from the above synthetic example to a human motion data set \cite{hodgins2015cmu}  to illustrate data topology exhibited in the latent space. In particular, the data set consists of (skeleton) gait sequences, which show a periodic pattern naturally. We apply the same preprocessing and train/test split as in~\cite{chen2015efficient,connor2019representing} and investigate the walking sequences of subject 15. After auto-encoding, we show in Figure~\ref{fig:gait}(a) the distance between consecutive frames in the latent space. The distances are expected to be similar since the discretization of motion is uniform. It is indeed the case for CAE. For plain auto-encoder, however, the cyclic walking pattern is broken in the latent space, because topology is not preserved by using a single chart (akin to the behavior that a circle is broken when the parameterization crosses $z=2\pi$ in Figure~\ref{fig:Manifold}. This phenomenon is exhibited by the periodically large distances seen in the plot. The VAE behavior is somewhere in-between, indicating that imposing structural priors to the latent space helps, despite being less effective than our manifold structure. Note that all three architectures use approximately the same number of parameters and their reconstruction error, as shown in Figure~\ref{fig:gait}(b,c),  also favors CAE (see a supplementary video for comparisons ). 
\begin{figure}[h]
  \centering
    \subfigure[]{
    \includegraphics[width=.30\linewidth]{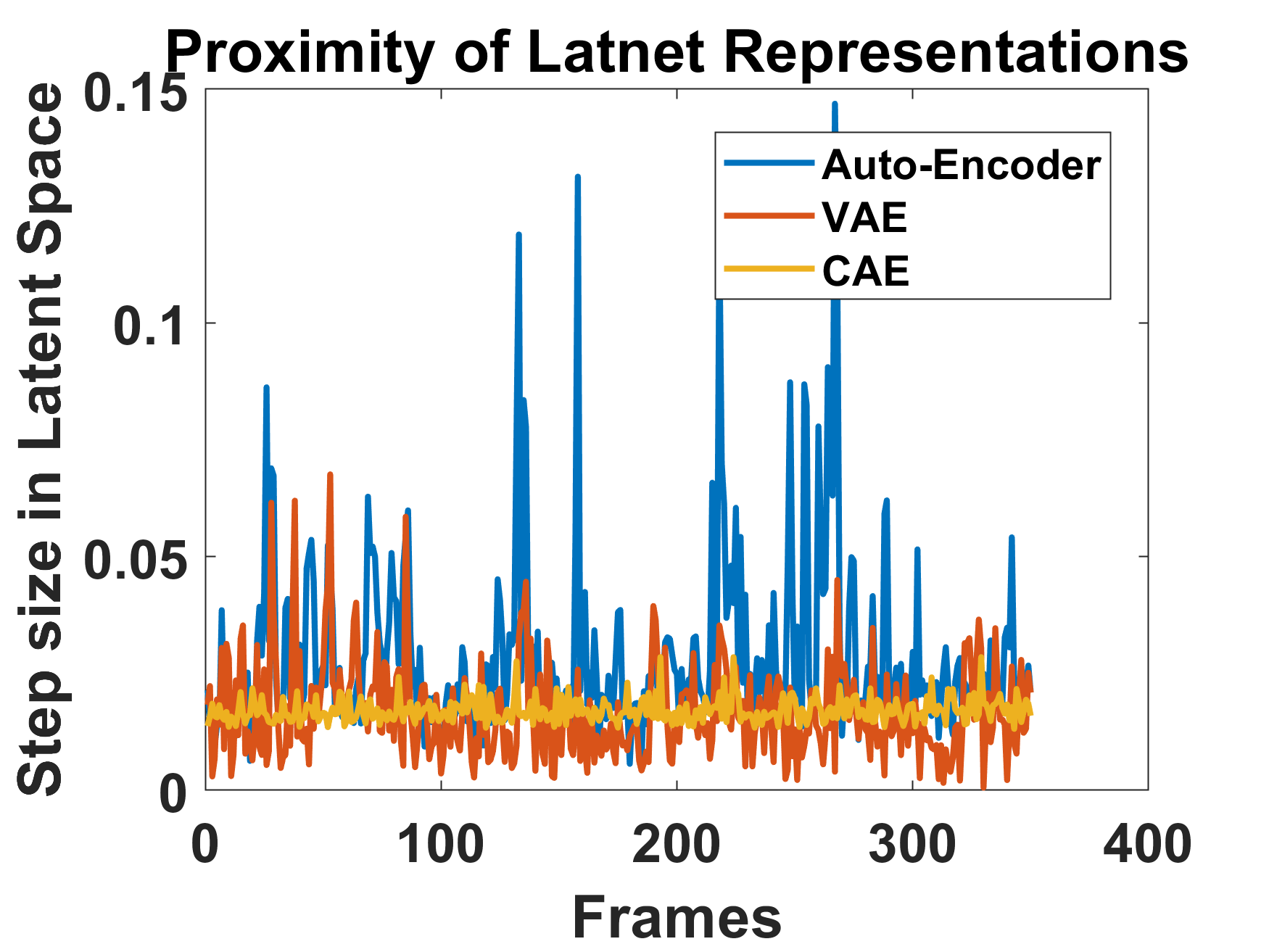}}
  \subfigure[]{
    \includegraphics[width=.30\linewidth]{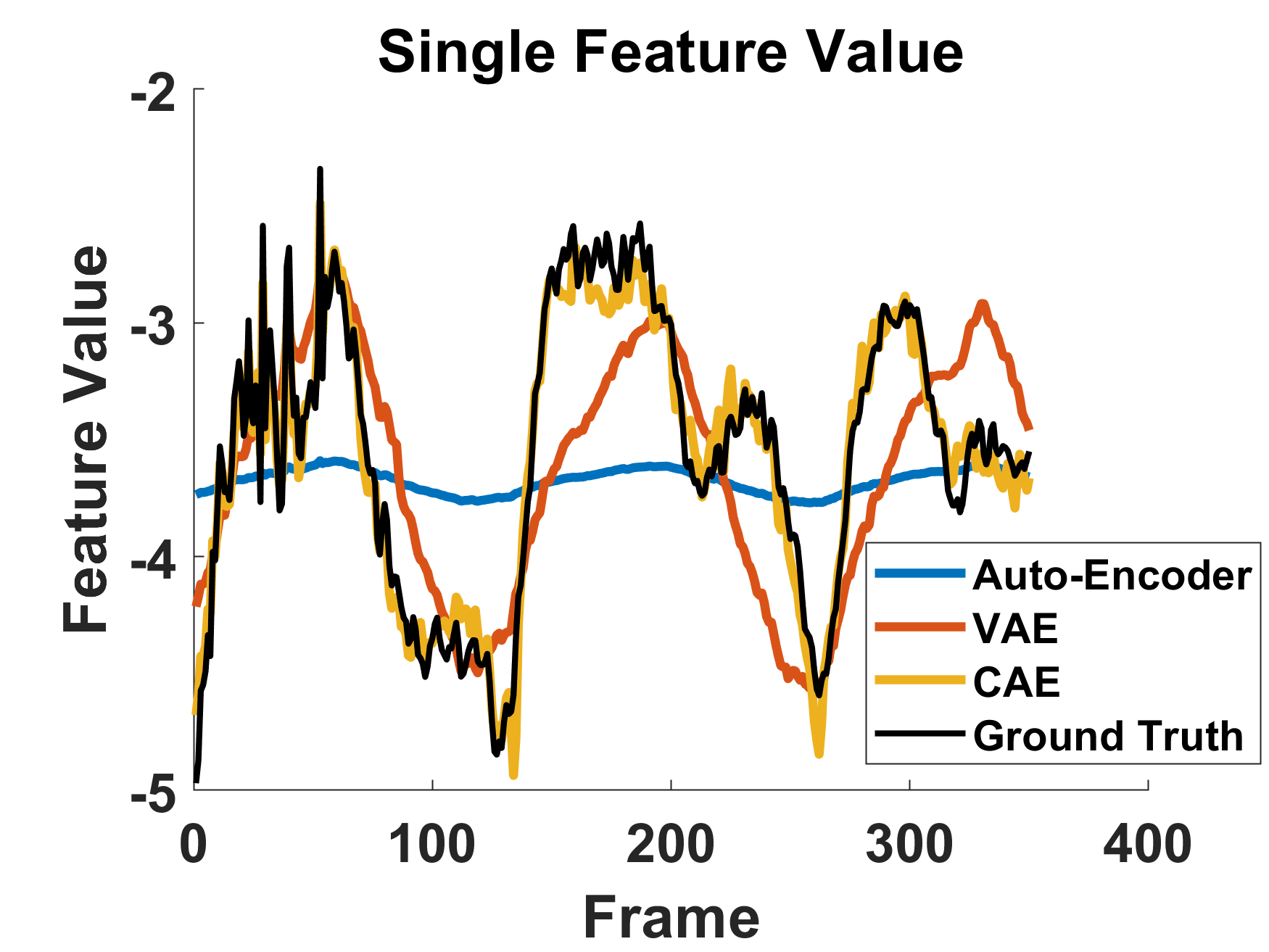}}
  \subfigure[]{
    \includegraphics[width=.30\linewidth]{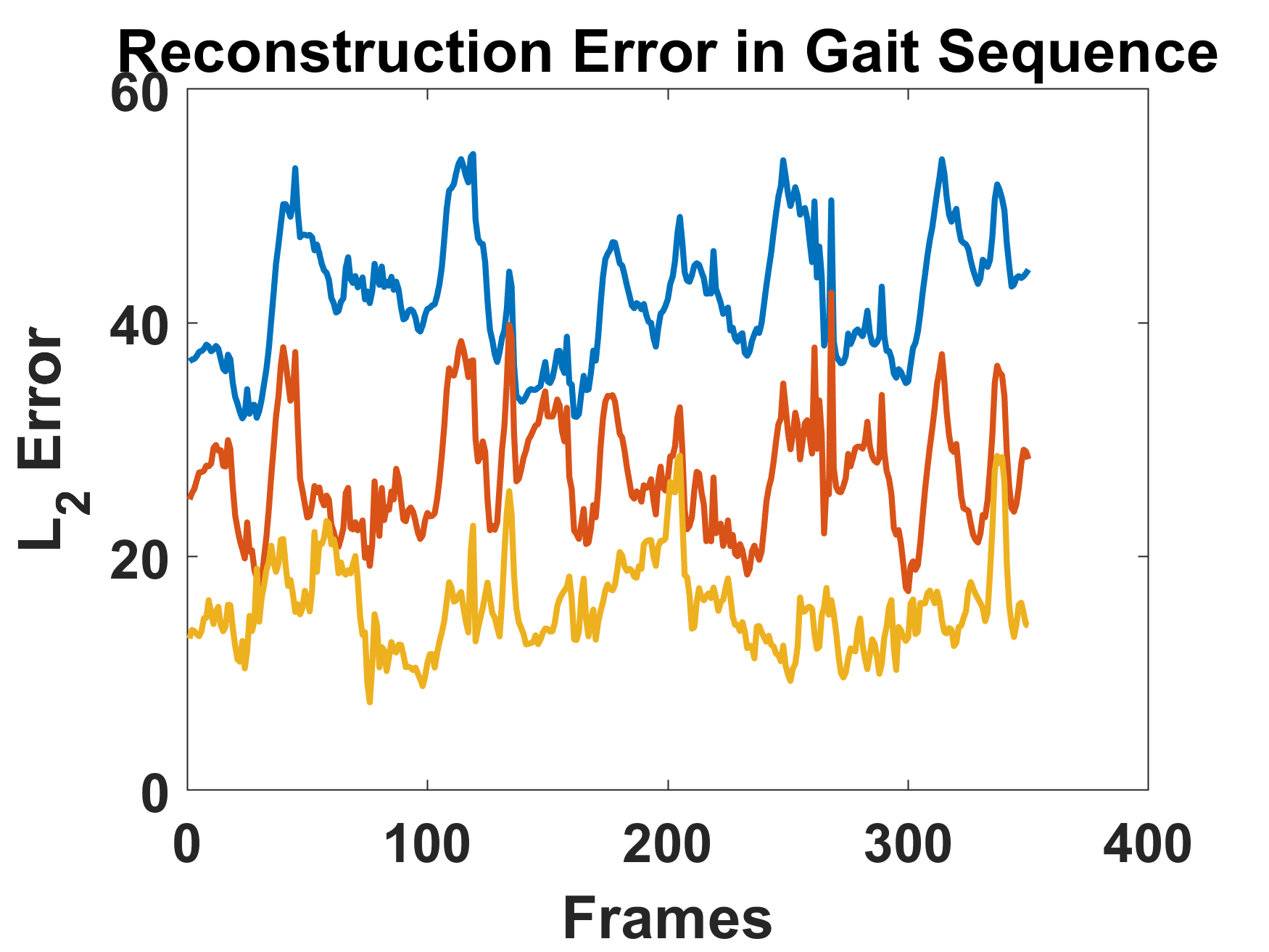}}
  \caption{Auto-encoding human motion sequence.  (a): Distance between consecutive frames in the latent space. (b): Value of a single feature. (c): Reconstruction error for all features. }
  \label{fig:gait}
\end{figure}

\textbf{Reconstruction, faithfulness, and coverage.}
The preceding experiment indicates the importance of structural prior. We more extensively compare CAE with VAE here, on MNIST \cite{lecun2010mnist} and fashion MNIST \cite{DBLP:journals/corr/abs-1708-07747}. In addition to the usual reconstruction error, we define two complementary metrics; both require a uniform sampling in the latent space to make sense. The first one, named \emph{unfaithfulness}, is the distance of a randomly generated sample from the training set. A smaller value means closer to the data manifold and hence less unfaithful (i.e., more faithful). The second metric, named \emph{coverage}, is the ratio between the number of distinct nearest training examples and the number of latent space samples. A high coverage is desirable; otherwise, some training examples (modes) are missed by the latent space. See supplementary material (Section~\ref{sec:eval}) for formal definitions.

From Table~\ref{tab:reconstruction}, one clearly sees that using the same latent dimension, VAE generally performs better with more parameters. Similarly, CAE also tends to yield better results with more parameters. However, for the same level of parameter count, CAE almost always outperforms VAE. We give more results through varying the latent dimension and the number of parameters in the supplementary material.

\begin{table}[ht]
  \caption{Reconstruction error and other metrics on MNIST and fashion MNIST.}
  \label{tab:reconstruction}
\resizebox{\textwidth}{!}{
\begin{tabular}{|c|c|c|c|c|c|c|}
\hline
\textbf{Model} & \textbf{Charts}                                         & \textbf{Latent Dim}                                      & \textbf{Param.}                                                                   & \textbf{Recon.  Error}                                                                                                  & \textbf{Unfaithfulness}                                                                                           & \textbf{Coverage}                                                                                         \\ \hline
\multicolumn{7}{|c|}{\textbf{MNIST}}                                                                                                                                                                                                                                                                                                                                                                                                                                                                                                                                              \\ \hline
VAE            & \begin{tabular}[c]{@{}c@{}}1\\ 1\\ 1\\ 1\end{tabular}   & \begin{tabular}[c]{@{}c@{}}4\\ 64\\ 8\\ 64\end{tabular}  & \begin{tabular}[c]{@{}c@{}}893,028\\ 938,088\\ 2,535,028\\ 2,625,088\end{tabular} & \begin{tabular}[c]{@{}c@{}}0.0614 $\pm$ .002\\ 0.0512 $\pm$ .002\\ 0.0564 $\pm$ .001\\ 0.0391 $\pm$ .002\end{tabular}   & \begin{tabular}[c]{@{}c@{}}0.083 $\pm$ .021\\ 0.070 $\pm$ .011\\ 0.085 $\pm$ .008\\ 0.081 $\pm$ .011\end{tabular} & \begin{tabular}[c]{@{}c@{}}0.83 $\pm$ .01\\ 0.94 $\pm$ .01\\ 0.91 $\pm$ .00\\ 0.96 $\pm$ .01\end{tabular} \\ \hline
CAE            & \begin{tabular}[c]{@{}c@{}}4\\ 4\\ 32\\ 32\end{tabular} & \begin{tabular}[c]{@{}c@{}}4\\ 16\\ 16\\ 32\end{tabular} & \begin{tabular}[c]{@{}c@{}}601,452\\ 635,196\\ 2,610,120\\ 2,924,808\end{tabular} & \begin{tabular}[c]{@{}c@{}}0.0516 $\pm$ .001\\ 0.0409 $\pm$ .001\\  0.0290  $\pm$ .001\\ 0.0289 $\pm$ .002\end{tabular} & \begin{tabular}[c]{@{}c@{}}0.069 $\pm$ .019\\ 0.065 $\pm$ .018\\ 0.043 $\pm$ .012\\ 0.045 $\pm$ .011\end{tabular} & \begin{tabular}[c]{@{}c@{}}0.92 $\pm$ .01\\ 0.94 $\pm$ .01\\ 0.98 $\pm$ .01\\ 0.98 $\pm$ .01\end{tabular} \\ \hline
\multicolumn{7}{|c|}{\textbf{FMINST}}                                                                                             \\ \hline
VAE            & \begin{tabular}[c]{@{}c@{}}1\\ 1\\ 1\\ 1\end{tabular}   & \begin{tabular}[c]{@{}c@{}}8\\ 64\\ 8\\ 64\end{tabular}  & \begin{tabular}[c]{@{}c@{}}893,028\\ 938,088\\ 2,535,028\\ 2,625,088\end{tabular} & \begin{tabular}[c]{@{}c@{}}0.0575 $\pm$ .001\\ 0.0568 $\pm$ .003\\ 0.0474 $\pm$ .001\\ 0.0291 $\pm$ .006\end{tabular}   & \begin{tabular}[c]{@{}c@{}}0.016 $\pm$ .021\\ 0.029 $\pm$ .034\\ 0.014 $\pm$ .008\\ 0.021 $\pm$ .011\end{tabular} & \begin{tabular}[c]{@{}c@{}}0.80 $\pm$ .01\\ 0.95 $\pm$ .01\\ 0.92 $\pm$ .00\\ 0.92 $\pm$ .01\end{tabular} \\ \hline
CAE            & \begin{tabular}[c]{@{}c@{}}4\\ 4\\ 32\\ 32\end{tabular} & \begin{tabular}[c]{@{}c@{}}4\\ 16\\ 16\\ 64\end{tabular} & \begin{tabular}[c]{@{}c@{}}601,452\\ 635,196\\ 2,610,120\\ 3,554,184\end{tabular} & \begin{tabular}[c]{@{}c@{}}0.0409 $\pm$ .001\\ 0.0301 $\pm$ .001\\  0.0190  $\pm$ .001\\ 0.0177 $\pm$ .002\end{tabular} & \begin{tabular}[c]{@{}c@{}}0.010 $\pm$ .001\\ 0.010 $\pm$ .001\\ 0.016 $\pm$ .001\\ 0.007 $\pm$ .021\end{tabular} & \begin{tabular}[c]{@{}c@{}}0.90 $\pm$ .01\\ 0.90 $\pm$ .01\\ 0.97 $\pm$ .02\\ 0.97 $\pm$ .02\end{tabular} \\ \hline
\end{tabular}
}
\end{table}

\textbf{Further experiments.}
Due to space limitation, we leave many experimental results in the supplementary material. We give illustrative examples to show the limitation of VAE in covering double torus, the capability of CAE for complex topology, the effects of Lipschitz regularization, the mechanism of automatic chart removal, as well as the measurement of geodesics. We also apply CAE to MNIST and fashion MNIST, discuss reconstruction, visualize the manifolds, and compare quantitatively with VAE. These results collectively demonstrate the use of CAE and its effectiveness over plain auto-encoder and VAE.

\section{Conclusions}
We have proposed and investigated the use of chart based parameterization to model manifold structured data, through introducing multi-chart latent spaces along with chart transition functions. The multi-chart parameterization follows the mathematical definition of manifolds and allows one to use an appropriate dimension for latent encoding. We theoretically prove that multi-chart is necessary for preserving the data manifold topology and approximating it $\epsilon$-closely.  We prove a universal approximation theorem on the representation capability of CAE and provides estimations of training data size and network size. Numerically, we design geometric examples to analyze the behavior of the proposed model and illustrate its advantage over plain auto-encoders and VAEs. We also apply the model to real-life data sets (human motion, MNIST, and fashion MNIST) to illustrate the manifold structures under-explored by existing auto-encoders.

\bibliography{refrences1}
\bibliographystyle{unsrt}

\newpage

\appendix

\section{Proofs}
\label{sec:proofs}

\paragraph{Proof of Theorem \ref{thm:faithfulrep}}We prove that $\M$ is homeomorphic to $\Z$ by showing that $\E$ is a homeomorphism, i.e., $E$ is one-to-one, on to, and $\E^{-1}$ is also continuous.   First, $\E$ is onto by the definition. Second, Assume that there are $x_1\neq x_2\in\M$ such that $\E(x_1) = \E(x_2) =z$, then $\|\D\circ\E(x_1) - x_1\| \leq \epsilon < \tau(\M)$ and $\|\D\circ\E(x_2) - x_2\| \leq \epsilon < \tau(\M)$ as the auto-encoder is $\epsilon$-faithful representation of $\M$. This contradicts with the definition of the reach $\tau(\M)$. Thus, $\E$ is a one-to-one map. Third, since $\E$ is bijective, we can havea well-defined $\E^{-1}$. Note that $\E$ is a continuous map from a compact space $\M$ to a Hausdorff space $\Z$. Any closed subset $C\subset \M$ is compact, thus $\E(C)$ is compact which is also closed in the Hausdorff space $\Z$. Thus, $\E$ is a closed map which maps a closed set in $\M$  to a closed set in $\Z$. By passage to complements, this implies pre-images of any open set under $\E^{-1}$ will be also open. Thus, $\E^{-1}$ is continuous. 

Similarly, $\D$ is also one-to-one, otherwise, there exist $z_1\neq z_2$ satisfying $\D(z_1) = \D(z_2)$. Since $\E$ is a homeomorphism. We can find $x_1\neq x_2$ such that $\E(x_1) = z_1$ and $\E(x_2) = z_2$. From the definition of $\epsilon$-faithful representation. We have  $\|\D\circ\E(x_1) - x_1\| \leq \epsilon < \tau(\M)$ and $\|\D\circ\E(x_2) - x_2\| \leq \epsilon < \tau(\M)$ which contradicts with the definition of $\tau(\M)$. Thus, $\D$ is a homeomorphism from $\Z$ to $\D(\Z)$ based on the same argument as before. 

Last but not least, if $\M$ is not contractible, it will not be homeomorphic to a simply connected domain. This concludes the proof. 
\qed
\newline

\begin{definition}[Simplicial complex]
\label{def:simplex}
A $d$-simplex $S$ is a $d$-dimensional convex hull provided by convex combinations of  $d+1$ affinely independent vectors $\{v_i\}_{i=0}^d \subset \real^m$. In other words, $\displaystyle S = \left\{ \sum_{i=0}^d \xi_i v_i ~|~ \xi_i \geq 0, \sum_{i=0}^d \xi_i = 1 \right \}$. If we write $V = (v_1 - v_0,\cdots,v_d - v_0)$, then $V$ is invertible and $S = \left\{v_0 + V\beta ~|~ \beta \in\real^m, \beta \in\Delta \right\}$ where $\Delta = \left\{\beta\in\real^d~|~ \beta \geq 0, \vec{1}^\top \beta \leq 1 \right \}$ is a template simplex in $\real^d$. The convex hull of any subset of $\{v_i\}_{i=0}^d$ is called a {\it face} of $S$. A simplicial complex $\displaystyle \mathcal{S} = \bigcup_\alpha S_\alpha$ is composed with a set of simplices $\{S_\alpha\}$ satisfying: 1) every face of a simplex from $\S$ is also in $\S$; 2) the non-empty intersection of any two simplices $\displaystyle S _{1},S _{2}\in \S$ is a face of both $S_1$ and $S_2$. For any vertex $v \in \S$, we further write $\N^1(v) = \{a~|~ v\in S_\alpha\}$ and $\displaystyle \S^1(v) = \bigcup_{\alpha\in\N^1(v)} S_\alpha $, the first ring neighborhood of $v$.
\end{definition}

\begin{theorem}
\label{thm:simplexfun}
Given a $d$-dimensional simplicial complex $\mathcal{S}= \bigcup_\alpha \S_\alpha \subset\real^d$ with $n$ vertices $\{v_\ell\}_{\ell=1}^n$ where each  $S_\alpha$ is a $d$-dimensional simplex. Then, for any given piecewise linear function $f:\S \rightarrow \mathbb{R}$ satisfying $f$ linear on each simplex, there is a ReLU network representing $f$ exactly. Moreover, this neural network has $n(K(d+1) + 4(2K-1)) + n $ paremeraters and $\log_2 (K) + 2$ layers, where $K = \max_i |\N(v_i)|$ which is bounded above by the number of total $d$-simplices in $\S$.
\end{theorem}
\begin{proof}
We first show a hat function on $\S$ can be represented as a neural network. Given a vertex $v\in\{v_\ell\}$, let $\displaystyle \S^1(v) = \bigcup_{i\in\N^1(v)} S_i $ be the first ring neighborhood of $v$. Let $\Delta = \left\{\beta\in\real^d~|~ \beta \geq 0, \vec{1}^\top \beta \leq 1 \right \}$ be a template simplex in $\real^d$ and write $S_i = \left\{ v + V_i ~\beta ~|~ \beta\in\Delta \right\}$ where $V_i\in \real^{d\times d}$ is determined by the vertices of $S_i$ and invertible. Let us write $F_i = \left\{ v + V_i ~\beta ~|~ \beta \geq 0 , \vec{1}^\top\beta = 1 \right\} $ and $\displaystyle \bigcup_{i\in\N^1(v)} F_i$ forms the boundary of the first ring  $\S^1(v)$. We consider the following one-to-one correspondence between a point $x\in\real^d$ and its coordinates $\beta$ on the simplex $S_i$: 
\begin{equation}
    T_i:\real^d \rightarrow \real^d,  \qquad x \mapsto  \beta_i = T_i(x) =  W_i x + b_i ,  \qquad  \forall i\in\N^1(v)
\end{equation}
where $W_i = V_i^{-1}, b_i = -V_i^{-1} v$.  Meanwhile, $\beta_i$ provides a convenient way to check if $x\in S_i$, namely, $x\in S_i \Leftrightarrow \beta_i = T_i(x) \in\Delta$. We define the following function $\eta_{v}:\S\rightarrow \real $:
\begin{equation}
    \eta_{v}(x) = \max \left\{ \min_{i\in\N^1(v)} \left\{ 1 - \vec{1}^\top (W_i x + b_i) \right\}, 0 \right \}
\end{equation}
We claim that $\eta_{v}$ is a pyramid (hat) function supported on $\S(v)$, namely, $\eta_{v}$ is a piecewise linear function satisfying:
\begin{equation}
\label{eqn:hatfun1}
    \eta_{v}(x) = \left\{ \begin{array}{cc}
        1 & \text{if} \quad x = v \\
        1  - \vec{1}^\top (W_i x + b_i), & \quad\quad  \text{if} \quad x \in S_i \quad \text{for some}~ i\in\N^1(v) \\
        0 &   \text{if}  \quad x\in \S -  \bigcup_{i\in\N^1(v)} S_i
    \end{array}\right.
\end{equation}
First, it is easy to see that $T_i(v) = 0, i = 1,\cdots,K$ which yields $\eta_{v}(v) = 1$. Second, assume $x\in S_i$, then $\beta_i = T_i(x) \in \Delta$ and $1 -  
\vec{1}^\top \beta_i \geq 0$. Consider the coordinates $\beta_j = T_j(x)$ of $x$ on $S_j, j (\neq i)\in\N^1(v)$. For those components of $\beta_j$ along the intersection edges of $S_i$ and $S_j$, they are exactly the same as the corresponding components of $\beta_i$. For those components of $\beta_j$ associated with directions along the non-intersection edges, they are negative. Thus, we have $\vec{1}^\top \beta_i \geq \vec{1}^\top \beta_j, \forall j\in\N^1(v)$. This implies $0\leq 1 -  
\vec{1}^\top \beta_i \leq 1 - \vec{1}^\top \beta_j, \forall j\in\N^1(v)$. Therefore, $\eta_{v}(x) = 1 - \vec{1}^\top(W_i x + b_i)$. In addition, it is straightforward to check $\eta_{v}(x) = 0, \forall x\in F_i$. Third, if $x\in  \left\{ v + V_i ~\beta ~|~ \beta \geq 0, \vec{1}^\top\beta> 1  \right\}$, we have $1 -  \vec{1}^\top \beta_i <0$, thus $\eta_{v}(x) = \max\{1 -  \vec{1}^\top \beta_i, 0\} = 0$. 
It is easy to see that 
\begin{equation}
\label{eqn:min}
\displaystyle \min\{ a, b\} = \frac{1}{2} (\mathrm{ReLu}(a+b) - \mathrm{ReLu}(a-b) - \mathrm{ReLu}(-a+b)  - \mathrm{ ReLu}(-a-b)),
\end{equation}
which means that $\min \{a, b\}$ can be represented as a 2-layer ReLU network. Based on equations \eqref{eqn:hatfun1} and \eqref{eqn:min} inductively, it straightforward to show that $\eta_{v}$ can be represented as a DNN with at most $\log_2(|\N^1(v)|) + 1$ layers and at most $ (|\N^1(v)|(d +1)+ 4(2(|\N^1(v)|-1)$ parameters according to Lemma~\ref{lemma:min} from~\cite{arora2016understanding}. Note that we can write $\displaystyle f(x) = \sum_{\ell} f(v_\ell) \eta_{v_\ell}(x)$; therefore, $f$ can be written as a DNN with at most $\log_2(K) + 2$ layers and at most $ n(K(d+1) + 4(2K-1)) + n $ parameters, where $K = \max_i |\N^1(v_i)|$ which is bounded above by the number of total $d$-simplices in $\S$. This concludes the proof.
\end{proof}

\begin{remark}
We remark that our construction is different from the construction used in~\cite{arora2016understanding}, which does not have estimation the number of parameters since it relies on a hinging hyperplane theorem in~\cite{wang2005generalization}.
\end{remark}

\begin{lemma}[Lemma D.3 in~\cite{arora2016understanding}]Let $f_1,\cdots,f_m$ be functions that can each be represented by ReLu DNNs with $k_i+1$ layers and $s_i$ parameters, $i = 1,\cdots, m$. Then, the function $g(x) = \min\{f_1(x),\cdots,f_m(x)\}$ can be represented by a ReLu DNN of at most $\max\{k_1,\cdots,k_m\} + \log_2(m) + 1$ layers and at most $s_1+\cdots+s_m + 4(2m - 1)$ parameters. 
\label{lemma:min}
\end{lemma}


\paragraph{Proof of Theorem \ref{thm:localchart}}
We begin with constructing a neural network on a given chart $\M_r(p) =  \{x\in\M~|~ d(p,x) \leq \gamma\}$. Let $\T_{p,r}\M = \{v\in\T_p\M~|~ \|v\| \leq r\}$. Since $\M$ is compact, then the exponential map $\exp_p (v):\T_{p,r}\M \rightarrow \M_r(p), v\mapsto \gamma_v(1) $ is one-to-one and onto where $\gamma_v(t)$ is a geodesic curve satisfying $\gamma_v(0) = p, \dot{\gamma}_v(0) = v$. We write the inverse of $\exp_v$ as the logarithmic map $\log_p(x): \M_r \rightarrow \T_{p,r}\M $.  In the rest of the proof, we will construct an encoder $\E$ to approximate $\log_p$ and a decoder $\D$ to approximate $\exp_p$ based on the training set $X= \{x_i\}_{i=1}^n$. We define $\{z_i\}_{i=1}^n = \{\log_{p} (x_i)\}_{i=1}^n \in \Z$ as the corresponding latent variables of $X$. Note that $\{z_i\}_{i=1}^n$ are sampled on a bounded domain $\T_{p,r}\M$; thus there exists a simplicial complex for $\{z_i\}_{i=1}^n$ through a Delanuay triangulation $\S = \bigcup_{\alpha=1}^T S_\alpha$ with $T = O (n^{\lceil d/2\rceil })$~\cite{bern1995dihedral}. Here, each $S_\alpha$ is a $d$-dimensional simplex whose vertices are $d+1$ points from $\{z_i\}_{i=1}^n$. From the one-to-one correspondence between $\{z_i\}_{i=1}^n$ and $\{x_i\}_{i=1}^n$, we can have a $d$-simplex $\bar{S}_\alpha$ by replacing vertices in $S_\alpha$ as the corresponding $x_i\in X$. This provides a simplicial complex $\bar{\S} = \bigcup_{\alpha=1}^T \bar{S}_\alpha$. Note that each vertex of $\bar{\S}$ is on $\M_r$; therefore, $\bar{\S}$ is a simplicial complex approximation of $\M_r(p)$. We define $\Z = \S = \bigcup_{\alpha=1}^T S_\alpha$ which is essentially a $d$-dimensional ball in $\real^d$. It is also  straightforward to define a simplicial map,
\begin{equation}
 F: \Z = \S\rightarrow \bar{\S}\subset\real^m, \qquad F(z) = \sum_{i=0}^d\xi_i x_{\alpha_i} \quad \text{for} \quad z = \sum_{i=0}^d\xi_i z_{\alpha_i}\in S_\alpha .
\end{equation} 
Here, $F$ maps $z_i$ to $x_i$ and piecewise linearly spend the rest of the map. According to Theorem~\ref{thm:simplexfun}, each component of $F$ can be represented as a neural network. Therefore, $F$ can be represented by a neural network $\D$ with at most $\displaystyle mn(n^{\lceil d/2\rceil }(d+1) + 4(2n^{\lceil d/2\rceil }-1)) + mn = O(mdn^{1+d/2})$ parameraters and $\displaystyle \lceil d/2\rceil \log_2 (n) + 2 = O(\frac{d}{2} \log_2(n))$ layers.\footnote{We remark that this estimation is not sharp since we overestimate $\max_i |\N^1(z_i)|$ using the total number of simplices. We conjecture that this number $\max_i |\N^1(z_i)|$ should constantly depend only on $d$. This is true for Delaunay triangulation to points distributed according to a Poisson process in $\real^d$~\cite{dwyer1991higher}. If this conjecture is true, then the number of parameters has order $O(mdn)$ which will improve the parameter size as $O(\epsilon^{-d})$. }

Next, we construct a decoder $\D:\M_r(p) \rightarrow \Z$. We first construct a projection from $\M_r(p)$ to its simplicial approximation $\bar{\S}$. We write $\{x_{\alpha_0}, \cdots, x_{\alpha_d}\} \subset X$ being $d+1$ vertices in a simplex $\bar{S}_\alpha$. For convenience, we write $V_\alpha = \{x_{\alpha_0} + X_\alpha \beta~|~ \beta\in\real^d\}$ and each $\bar{S}_\alpha = \{x_{\alpha_0} + X_\alpha \beta~|~ \beta\in\Delta\}$ where $\Delta$ is the template $d$-simplex used in Definition~\ref{def:simplex}. Note that $X_\alpha = (x_{\alpha_1}-x_{\alpha_0}, \cdots, x_{\alpha_d}-x_{\alpha_0})\in\real^{m\times d}$ is a full rank matrix.
We define the projection operator:
\begin{equation}
\text{Proj}_\alpha : \M_r(p) \rightarrow V_\alpha , \qquad x \mapsto \text{Proj}_\alpha(x) = X_\alpha X_\alpha^{\dagger} (x - x_{\alpha_0}) + x_{\alpha_0}
\end{equation}
where  $X_\alpha^{\dagger} = (X_\alpha^\top X_\alpha)^{-1} X_\alpha^\top$ is the Moore-Penrose pseudo-inverse of $X_\alpha$. It is clear to see that $X_\alpha^{\dagger} (x - x_{\alpha_0})$ provides coordinates of $\text{Proj}_{\alpha}(x)$ in the simplex $\bar{S}_\alpha$. Similar as the construction used in the proof of Theorem~\ref{thm:simplexfun}, for each $x_i \in X$ surrounded by $\{\bar{S}_\alpha\}_{\alpha\in\N^1(x_i)}$, we construct the function
\begin{equation}
    \eta_{x_i}(x) =  \chi(\|x - x_i\|^2) \max \left\{ \min_{ \alpha \in\N^1(x_i)} \left\{ 1 - \vec{1}^\top X_\alpha^{\dagger} (x - x_{\alpha_0})  \right\}, 0 \right \}
\end{equation}
where 
$\chi$ serves as an approximation of a indicator function with $0<\mu \ll \epsilon$ and
\begin{equation}\displaystyle \chi(t) = \left\{\begin{array}{cc} 1, & \text{if} \quad  0\leq t\leq \epsilon^2 + \mu \\
1 + \frac{1}{\mu} (\epsilon^2 + \mu - t), & \text{if} \quad  \epsilon^2 + \mu  \leq t \leq \epsilon^2 + 2\mu \\
 0, & \text{if} \quad  t \geq \epsilon^2 +  2\mu .
\end{array} \right. \end{equation}
Since $X$ is $\epsilon/2$-dense, we have $\max_{\alpha\in\N^1(x_i)}\|x_i - x_\alpha\| \leq \epsilon$. Thus the indicator function $\chi$ restricts $x$ in the first ring of $x_i$. In addition, for any $x$ in the support of $\chi(\|x - x_i\|^2)$, there exists a unique simplex $\bar{S}_\alpha$ in the first ring of $x_i$ such that $X_\alpha^{\dagger} (x - x_{\alpha_0})\in\Delta$; otherwise this contradicts with $\epsilon < \tau$. Using a similar argument in Theorem~\ref{thm:simplexfun}, one can also show that 
\begin{equation}
\label{eqn:hatfun2}
    \eta_{x_i}(x) = \left\{ \begin{array}{cc}
        1 & \text{if} \quad x = x_i \\
        1 - \vec{1}^\top X_\alpha^{\dagger} (x - x_{\alpha_0}) , & \quad \text{if} \quad x\in \bar{S}_\alpha \quad \text{for some}~ \alpha\in\N^1(x_i) \\
        0 &   \text{otherwise}.
    \end{array}\right.
\end{equation}
It is straightforward to check that $\chi(t) =  \frac{1}{\mu}  \text{ReLu}(-t + \epsilon^2 + 2\mu) - \frac{1}{\mu}  \text{ReLu}(-t + \epsilon^2 + \mu)$. Therefore, $\eta_{x_i}$ can be represented as the neural network with at most $\log_2(|\N^1(v)|) + 1$ layers and at most $ (|\N^1(v)|(m +1)+ 4(2(|\N^1(v)|-1)$ parameters. Note that this neural network is not a feed forward ReLU network as we request the network compute multiplication for $\|x - x_i\|^2$ and multiplication between $\chi$ and output of ReLu.\footnote{We remark that multiplication can also be approximated by ReLU network according to Propositions 2 and 3 in \cite{yarotsky2017error}, where the author showed that multiplications can be approximated by a ReLU network with any error $\epsilon$ with layers and number of parameters bound above by $O(C\log(1/\epsilon)$). We can adjust the buffer zone by controlling the parameter $\mu$ to accommodate this approximation error. } We define the encoder $\E(x) = F^{-1}\circ\text{Proj}_\alpha(x)$. Here $\bar{S}_\alpha$ is chosen as the closest simplex to $x$ and $F^{-1}(x) = \sum_{i=0}^d\xi_i z_{\alpha_i} \quad \text{for} \quad x = \sum_{i=0}^d\xi_i x_{\alpha_i}\in \bar{S}_\alpha$. Similar as approximation of $F$, we can use a neural network to represent $\E$ (as $d$ functions) with at most $ O(mdn^{1+d/2})$ parameters and $\displaystyle \lceil d/2\rceil \log_2 (n) + 2 = O(\frac{d}{2} \log_2(n))$ layers.

Now, we estimate the difference between $x$ and $\D\circ\E(x)$. From the above construction of $\E$ and $\D$, we have $\D\circ\E(x) = \text{Proj}_\alpha(x)$ where $\bar{S}_\alpha$ is chosen as the closest simplex to $x$.  Since $X$, the vertices of $\bar{\S}$, is an $\epsilon/2$-dense sample, this implies $\text{diam}(\bar{S}_\alpha) = \max \Big\{ \|x -y\|~|~ x,y\in \bar{S}_\alpha \Big\} \leq \epsilon$. Since the reach of $\M$ is $\tau$, the worst scenario becomes to compute the approximation error between a segment connecting two $\epsilon$-away points on a radius $\tau$ circle to itself. This error is $\tau - \sqrt{\tau^2 - (\epsilon/2)^2} \leq \epsilon$ since $\epsilon < \tau/2$.

As a byproduct of this constructive proof, we have $\D\circ\E(x_i) = x_i, i=1,\cdots,n$. This means the constructed $\E$ and $\D$ are global minimizers of the training loss $\displaystyle \frac{1}{n}\sum_{i=1}^n \|\D\circ\E(x_i) - x_i\|^2$. 

This concludes the proof.
\qed

\paragraph{Proof of Theorem \ref{thm:UMA}}
We first apply proposition 3.2 from Niyogi-Smale-Weinberger~\cite{niyogi2008finding} to obtain an estimation of the cardinality of the training set $X$ on $\M$ satisfying that $X$ is $\epsilon/2$-dense on $\M$ with probability $1-\nu$. We reiterate this proposition here:
\begin{prop}\label{Homology} {Niyogi-Smale-Weinberger}~\cite{niyogi2008finding}
Let $\M$ be a $d$-dimensional compact manifold with reach $\tau$. Let $ X = \{x_i\}_{i=1}^n$ be a set of $n$ points drawn in i.i.d. fashion according to the uniform probability measure on $\M$. Then with probability greater than $1-\nu$, we have that $X$ is $\epsilon/2$-dense ($\epsilon < \tau/2$) in $\M$, provided that
\begin{equation}\label{eqn:npts}
    n > \beta_1 \Big(\log (\beta_2) + \log (\dfrac{1}{\nu}) \Big),
\end{equation}
where $ \beta_1 = \dfrac{vol(\M)}{\cos^d\Big(\arcsin(\dfrac{\epsilon }{8 \tau} )\Big)vol(B^d_{\epsilon/4})}$ and $\beta_2 = \dfrac{vol(\M)}{\cos^d\Big(\arcsin (\dfrac{\epsilon }{16 \tau} )vol(B^d_{\epsilon/8})}$. Here $vol(B^d_\delta)$ denotes the volume of the standard $d$-dimensional ball of radius $\delta$. 
\end{prop}
Note that $\displaystyle vol(B^d_\delta) = \frac{\pi^{d/2}\delta^d}{\Gamma(1+ d/2)}$ and $\cos(\arcsin(\delta)) = \sqrt{1 - \delta^2}$. Plugging them in the above proposition yields $ \displaystyle \beta_1 = C~ \Big(\frac{\epsilon}{4}\Big)^{-d}\Big(1 - (\frac{\epsilon}{8\tau})^2\Big)^{-d/2} $ and $\displaystyle \beta_2 = C~\Big(\frac{\epsilon}{8}\Big)^{-d}\Big(1 - (\frac{\epsilon}{16\tau})^2\Big)^{-d/2}  $. It is clear that $n = O\Big(-d\epsilon^{-d}\log\epsilon\Big)$. 

Since $\M$ is a compact manifold, we cover $\M$ using $L$ geodesic balls as we used in Theorem \ref{thm:localchart}, i.e. $\M = \bigcup_{\ell=1}^L \M_r(p_\ell)$. One can control the radius such that $L \geq d$. We write restriction of training data set $X$ on each of $\M_\ell$ as $X_\ell = \M_r(p_\ell)\cap X$. Since $X$ is uniformly sampled on $\M$, thus $|X_\ell | = O(n/L)$. Based on the local Theorem \ref{thm:localchart}, each of $\M_r(p_\ell)$ has an $\epsilon$-faithful representation $(\Z_\ell, \E_\ell, \D_\ell)$ where each $\Z_\ell$ is a radius-$r$ standard ball in $\real^d$, both $\E_\ell$ and $\D_\ell$ have  $O(md (n/L)^{1+d/2}) = O(md \epsilon^{-d - d^2/2}(-\log^{1+d/2}\epsilon))$ parameters and $\displaystyle  O(\frac{d}{2}\log_2(n/L)) = O(-d^2\log_2\epsilon/2)$ layers.

To construct a latent space for $\M$, we consider a disjoint union $\bar{\Z} = \bigsqcup_\ell \Z_\ell$ and glue $\Z_\ell$ through an equivalence relation. Given $z_{\ell_1}\in\Z_{\ell_1}$ and $ z_{\ell_2} \in \Z_{\ell_2}$, we define $z_{\ell_1}\sim z_{\ell_2}$ if $\E_{\ell_1}^{-1} (z_{\ell_1}) = \E_{\ell_2}^{-1} (z_{\ell_2})$, then the latent space is defined as $\Z = \bar{\Z}/\sim$. This construction guarantees that $\Z$ is homeomorhphic to $\M$ which is compatible with the result from Theorem~\ref{thm:faithfulrep}. According the construction of $\D_\ell$, it is clear to see that if $z_{\ell_1}\sim z_{\ell_2}$, then $\D_{\ell_1}(z_{\ell_1}) = \D_{\ell_2}(z_{\ell_2})$. Therefore, the collection of encoders and decoders are well-defined. Since each of $(\Z_\ell, \E_\ell, \D_\ell)$ is $\epsilon$-faithful representation, $(\Z, \{\E_\ell\}, \{\D_\ell\})$ is also $\epsilon$-faithful. Overall,  encoders and decoders have $O(Lmd \epsilon^{-d - d^2/2}(-\log^{1+d/2}\epsilon))$ parameters and $\displaystyle   O(-d^2\log_2\epsilon/2)$ layers. This concludes the proof. 
\qed


\section{Additional Experiment Results}
\label{app:Experiments}
\subsection{Illustrative Examples}

\paragraph{VAEs Do Not Generalize for Double Torus.}
Figure~\ref{fig:vae} shows an experiment of VAEs with increasingly more parameters on data sampled from a double torus. The latent space dimension is set at two, the intrinsic dimension of the object. One sees that increasing the number of parameters in a VAE alone (without increasing the latent dimension) does not simultaneously produce good reconstruction and generalize. A latent space with a small dimension does not cover the entire manifold and a model with too many parameters overfits (the generated points may be far from the manifold). As a comparison, our CAE model using two-dimensional latent space can cover the data manifold very well (See Figure \ref{fig:Eight}).

\begin{figure}[h]
  \centering
  \includegraphics[width=.7\linewidth]{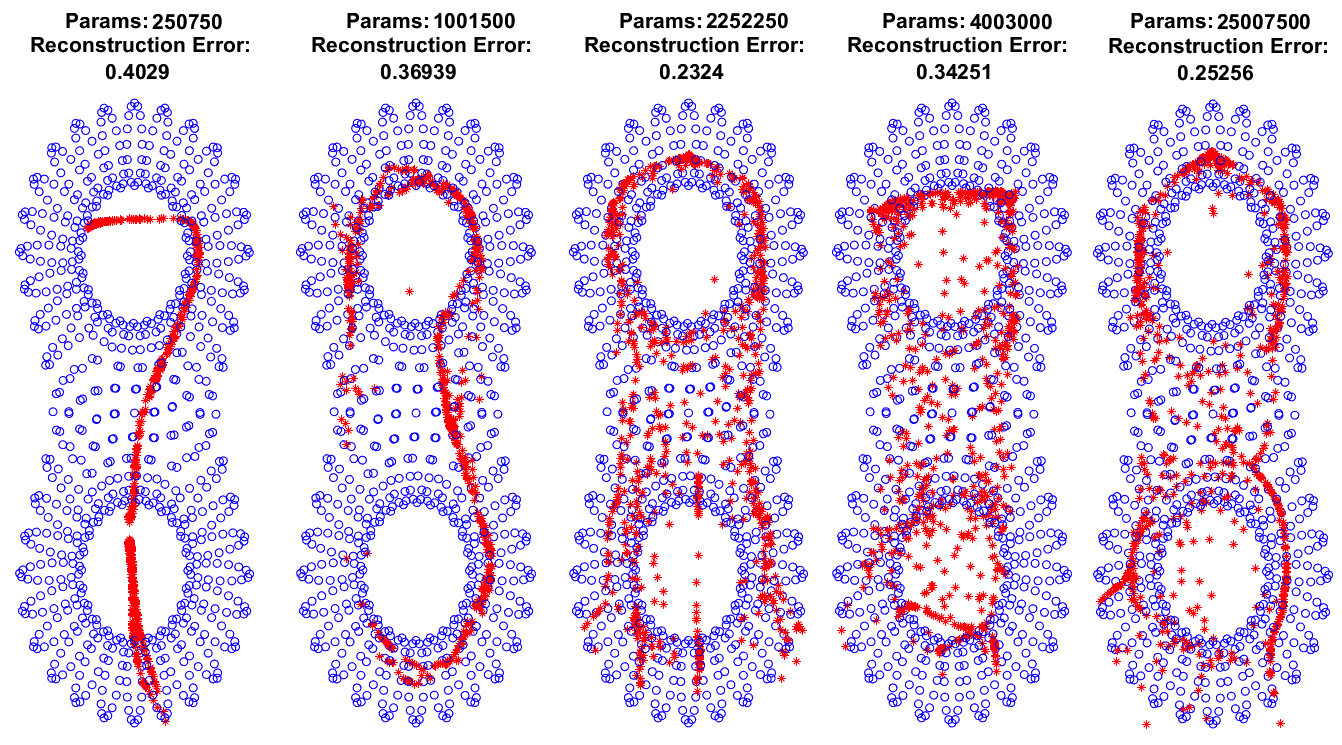}
  \caption{Increasingly overparametized VAEs with 2-dimensional flat latent space for data sampled from a double torus. Blue: training data. Red: generated data sampled from the latent space.}
  \label{fig:vae}
\end{figure}

\paragraph{Complex Topology.}
Beyond circles and spheres, CAE can handle increasingly more complex manifolds. Figure~\ref{fig:Gen3Results} shows a genus-3 manifold example, which is the surface of a pyramid with three holes. We use ten 2-dimensional charts to cover the entire manifold. Figure \ref{fig:Gen3Results} illustrates that new points generated by CAE stay close to the data manifold.
\begin{figure}[ht]
  \centering
  \includegraphics[width=0.6\linewidth]{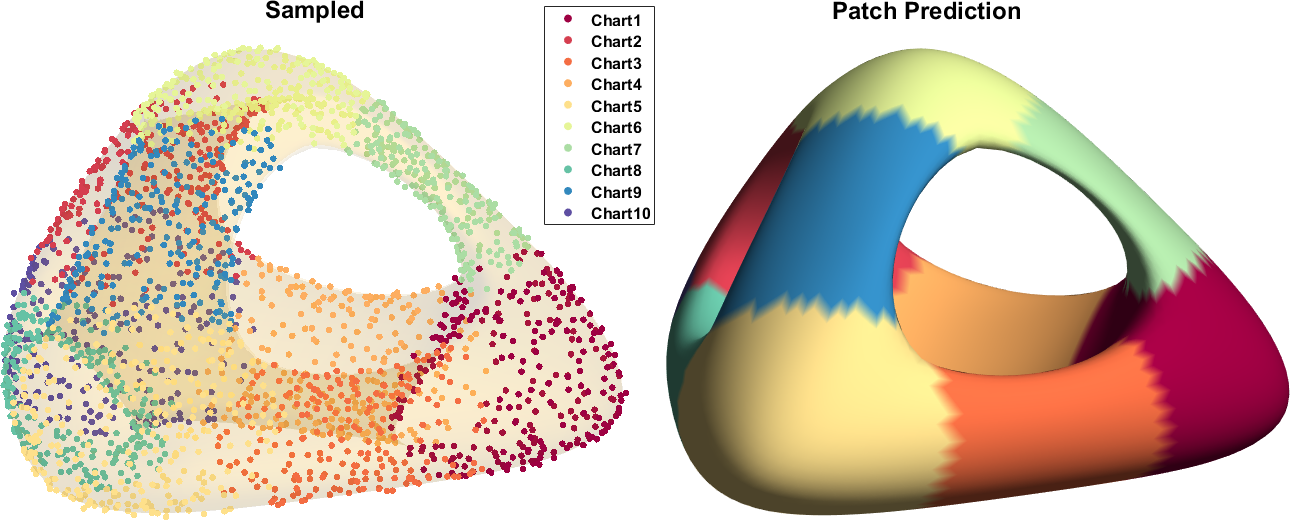}
  \vskip -0.1in
  \caption{Left: Points sampled from high probability regions. Right: Charts after taking max. }
  \label{fig:Gen3Results}
\end{figure}

\paragraph{Effects of Lipschitz Regularization.}
\begin{figure}[ht]
\centering 
  \includegraphics[width=.6\linewidth]{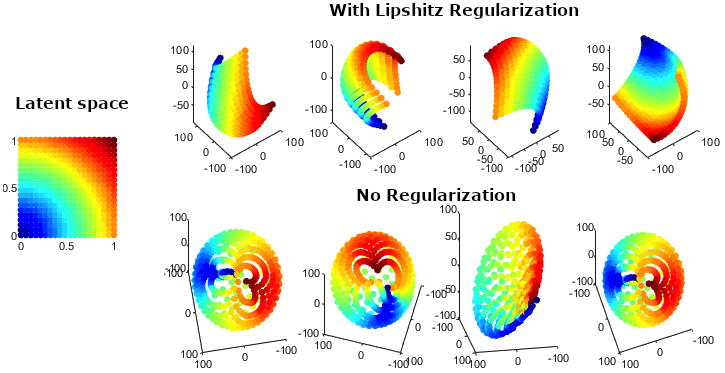}
  \caption{Left: Chart latent space. Top: Model with Lipschitz regularization. Bottom: Model without Lipschitz regularization.}
  \label{fig:SphereEncoders}
\end{figure}
In Section~\ref{sec:Reg} we mentioned the use of Lipschitz regularization as an important tool to stabilize training and encourage reasonable chart size. In Figure~\ref{fig:SphereEncoders} we show the result of auto-encoding a sphere in $\mathbb{R}^3$ using a 2-dimensional charted latent space. The top row shows the charts trained with Lipschitz regularization and bottom without regularization. One clearly sees that with Lipschitz regularization the charts are well localized, whereas without such regularization each chart spreads over the sphere, but none covers the entire sphere well.

\paragraph{Automatic Chart Removal.}
As discussed in Section~\ref{sec:Arch}, it is hard to know a priori the sufficient number of charts necessary to cover an unknown manifold. Hence, we propose over-specifying a number and relying on regularization to eliminate unnecessary charts. In Figure~\ref{fig:kill} we illustrate such an example. 
Pretrainining results in four charts but subsequent training removes two automatically.
\begin{figure}[ht]
\centering
  \includegraphics[width=.5\linewidth]{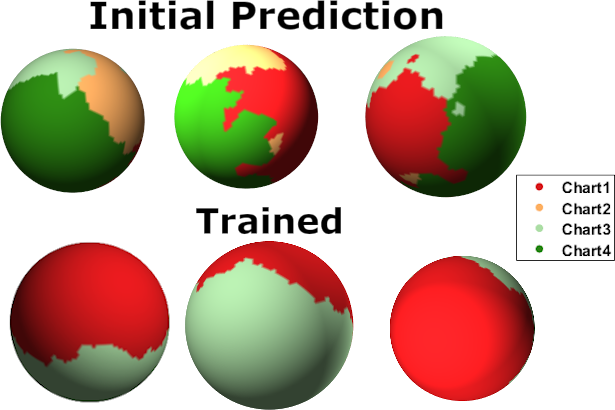}
  \caption{Top: Pre-trained charts. Bottom: Final charts after training. }
  \label{fig:kill}
\end{figure}

\paragraph{Measuring Geodesics.}
One advantage of CAE compared with plain auto-encoders and VAEs is that it is able to measure geometric properties of the manifold. In Figure~\ref{fig:geo}, we illustrate measuring geodesics as an example. To measure the geodesic distance of two points, we encode each point, connect them in the latent space, and sample points along the connection path. We then approximate the geodesic distance by summing the Euclidean distances for every pair of adjacent points. By increasing the number of sampling points, we can improve the approximation quality. The figure shows a few geodesic curves and their approximation error, decreasing with denser sampling.

\begin{figure}[ht]
  \centering
  \includegraphics[width=.6\linewidth]{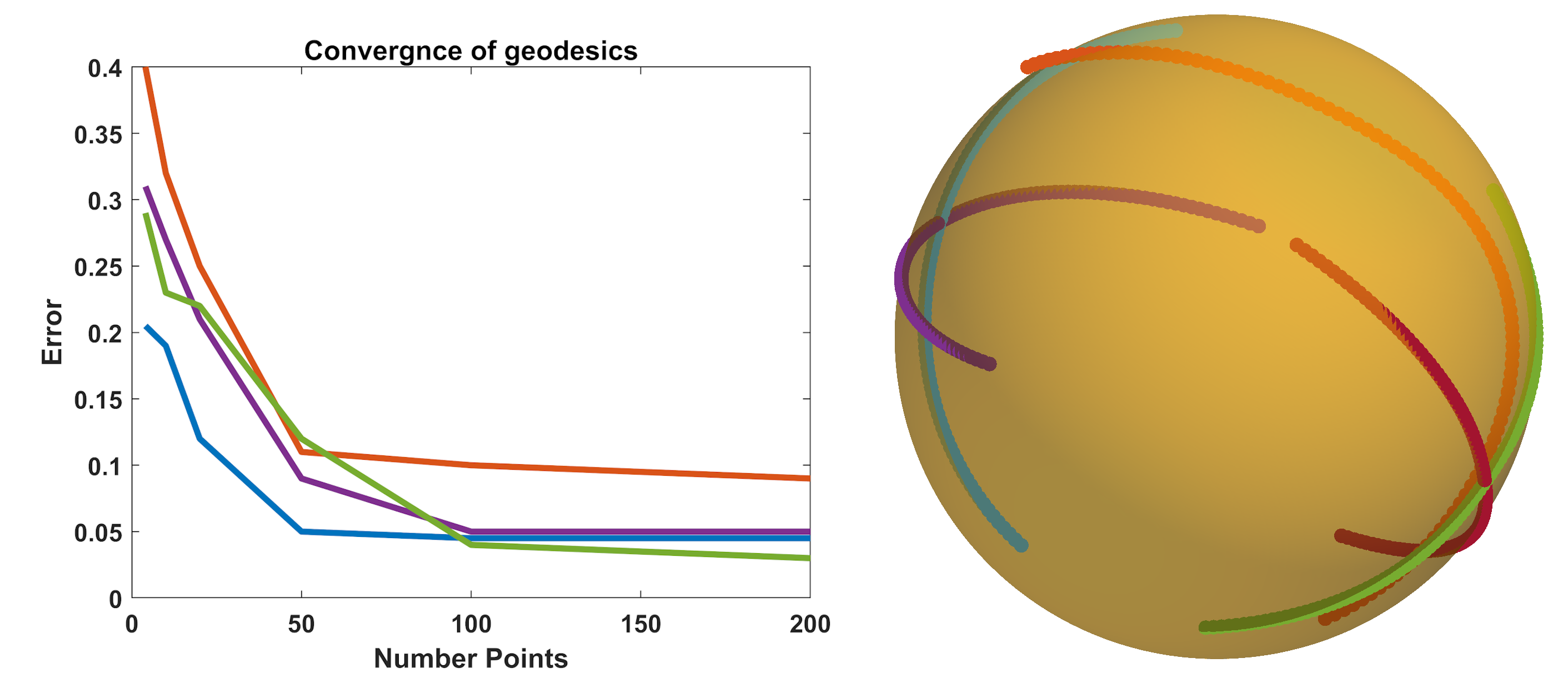}
  \vskip -0.1in
  \caption{Left: Geodesic approximation error v.s. number of points sampled in the latent space. Right: Geodesic curves generated from the chart decoders.}
  \label{fig:geo}
\end{figure}

\subsection{The MNIST and Fashion MNIST Manifolds}
In this subsection, we train a 4-chart CAE on MNIST and Fashion MNIST and explore the data manifold.

\paragraph{Decoder Outputs.}
Figure~\ref{fig:MNIST} illustrates several decoding results. Each column corresponds to one example. One finds that each chart decoder produces a legible digit, which may or may not coincide with the input. However, the maximum probability always points to the correct digit. Moreover, in some cases several chart decoders produce similar correct result (e.g., `7', `1', `9', and `6'), which indicates that the corresponding charts overlap in a region surrounding this digit.

\begin{figure}[ht]
  \centering
  \includegraphics[width=.6\linewidth]{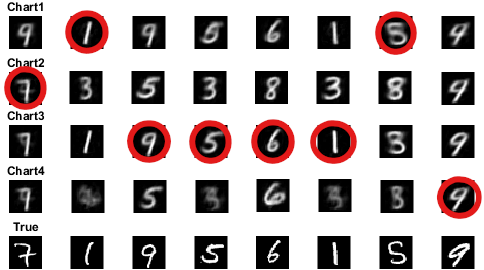}
  \caption{Decoder outputs for a few digit examples. The circled outputs receive the highest probability and serve as the final reconstruction.}
  \label{fig:MNIST}
\end{figure}

\paragraph{Visualization of the Manifolds.}
To understand the manifold structure globally, we visualize each chart in Figure~\ref{fig:FMNIST}. In these plots, we use t-SNE~\cite{maaten2008visualizing} to perform dimension reduction from the charted latent space to two dimensions. For MNIST, one sees that some digits are mostly covered by a single chart (e.g., purple 9) whereas others appear in multiple charts (e.g., navy blue 8). Similar observations are made for Fashion MNIST.

\begin{figure}[ht]
\begin{minipage}{.48\textwidth}
  \centering
  \includegraphics[width=.95\linewidth]{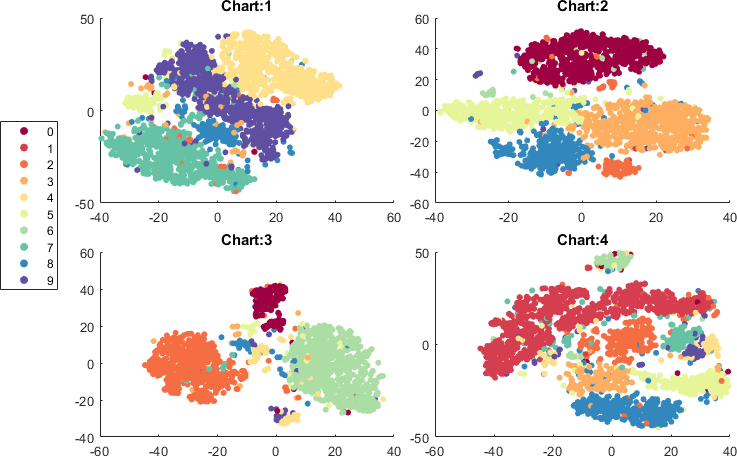}
\end{minipage}
\hspace{.2cm}
\begin{minipage}{.48\textwidth}
  \centering
  \includegraphics[width=.95\linewidth]{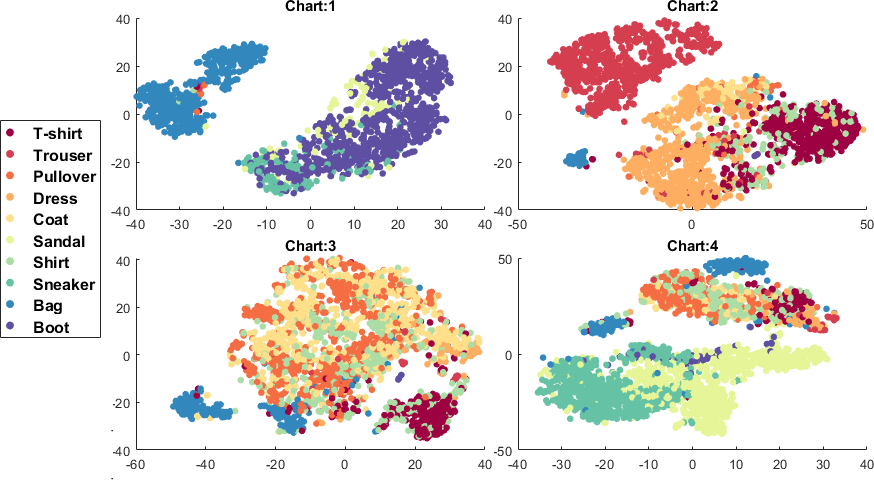}
\end{minipage}
\caption{Left: Visualization of the MNIST manifold by charts. Right: Visualization of the Fashion MNIST manifold by charts.}
  \label{fig:FMNIST}
\end{figure}

\subsection{Extensive Model Evaluation}\label{sec:eval}
To qualitatively evaluate the effectiveness of CAE, we conduct comparisons between VAE and CAE on three data sets: a sphere embedded in $\RR^{50}$, MNIST \cite{lecun2010mnist}, and Fashion MNIST \cite{DBLP:journals/corr/abs-1708-07747}. We define the following evaluation metrics to quantify the performance. 

\paragraph{Reconstruction Error}
Let $x$ be a data point in the test set $D_{test}$ and $y(x)$ be its reconstruction. Let there be $n$ test points. The reconstruction error is
\begin{equation}\label{eqn:recon}
\mathcal{E}_{recon}:= \frac{1}{n}\sum_{x \in D_{Test}} ||x-y||^2 .
\end{equation}

\paragraph{Unfaithfulness}
Let $\{z_i\}_{i=1}^\ell$ be a uniform sampling in the latent space and $\textbf{D}$ denote the decoder. Let $D_{test}$ be the training set. The unfaithfulness is
\begin{equation}\label{eqn:faithful}
\mathcal{E}_{unfaithful}= \frac{1}{\ell} \sum_{i=1}^\ell \min_{x \in D_{train} } \|x - \textbf{D} (z_i)\|^2 .
\end{equation}
We set $\ell=100$. The concept of unfaithfulness is complementary to the concept of novelty in deep generative models. Whereas novel samples are encouraged, this metric is concerned with how close the novel sample stays to the manifold. When the training set is sufficiently dense on the data manifold, newly generated data faraway from anything observed during training are unlikely to be realistic.

\paragraph{Coverage}
Let $\ell^*$ be the cardinality of the set
$
\{x^*~|~ x^* = \arg \min_{x \in D_{train} } \|x - \textbf{D} (z_i)\|^2 \}
$. 
Then, we define the coverage
\begin{equation}\label{eqn:coverage}
\mathcal{E}_{coverage} = \frac{\ell^*}{\ell}.
\end{equation}
A coverage score close to 1 indicates that the newly generated samples are well distributed on the manifold, whereas a score close to 0 indicates that the model may be experiencing mode collapse.

From Tables \ref{tab:MNIST} and \ref{tab:FMNIST}, one sees that at the same level of model complexity, CAE achieves the best results in all three metrics. Curiously, for VAEs and plain auto-encoders, the reconstruction error stays approximately the same regardless of the latent dimension; and only CAEs are able to reduce the reconstruction error through increasing the number of parameters.


\begin{table}[ht]
\caption{Extensive Comparison on MNIST.}
\label{tab:MNIST}
\resizebox{\textwidth}{!}{
\begin{tabular}{|c|c|c|c|c|c|c|}
\hline
\textbf{\begin{tabular}[c]{@{}c@{}}Model\\ \end{tabular}} & \textbf{\begin{tabular}[c]{@{}c@{}}\# of \\ Charts\end{tabular}}                                      & \textbf{\begin{tabular}[c]{@{}c@{}}Latent \\ dim\end{tabular}}                                   & \textbf{\begin{tabular}[c]{@{}c@{}}\# of \\  Param.\end{tabular}}                                                                                                                                 & \textbf{\begin{tabular}[c]{@{}c@{}}Recon.\\  Error\end{tabular}}                                                                                                                                                                                                                                                                & \textbf{Unfaithfulness}                                                                                                                                                                                                                                                                                             & \textbf{Coverage}                                                                                                                                                                                                                                                                     \\ \hline
\begin{tabular}[c]{@{}c@{}}Small\\ VAE\end{tabular}                     & \begin{tabular}[c]{@{}c@{}}1\\ 1\\ 1\\ 1\\ 1\end{tabular}                                             & \begin{tabular}[c]{@{}c@{}}4\\ 8\\ 16\\ 32\\ 64\end{tabular}                                        & \begin{tabular}[c]{@{}c@{}}893,028\\ 896,032\\ 902,040\\ 914,056\\ 938,088\end{tabular}                                                                                                           & \begin{tabular}[c]{@{}c@{}}0.0614 $\pm$ .002\\ 0.0607 $\pm$ .003\\ 0.0519 $\pm$ .001\\ 0.0564 $\pm$ .003\\ 0.0512 $\pm$ .002\end{tabular}                                                                                                                                                                                       & \begin{tabular}[c]{@{}c@{}}0.083 $\pm$ .021\\ 0.085 $\pm$ .016\\ 0.084 $\pm$ .021\\ 0.084 $\pm$ .023\\ 0.070 $\pm$ .011\end{tabular}                                                                                                                                                                              & \begin{tabular}[c]{@{}c@{}}0.83 $\pm$ .01\\ 0.81 $\pm$ .02\\ 0.87 $\pm$ .02\\ 0.91 $\pm$ .01\\ 0.94 $\pm$ .01\end{tabular}                                                                                                                                                            \\ \hline
\begin{tabular}[c]{@{}c@{}}Large\\ VAE\end{tabular}                     & \begin{tabular}[c]{@{}c@{}}1\\ 1\\ 1\\ 1\\ 1\end{tabular}                                             & \begin{tabular}[c]{@{}c@{}}4\\ 8\\ 16\\ 32\\ 64\end{tabular}                                        & \begin{tabular}[c]{@{}c@{}}2,535,028\\ 2,541,032\\ 2,553,040\\ 2,577,056\\ 2,625,088\end{tabular}                                                                                                 & \begin{tabular}[c]{@{}c@{}}0.0564 $\pm$ .001\\ 0.0525 $\pm$ .002\\ 0.0401 $\pm$ .001\\ 0.0414 $\pm$ .001\\ 0.0391 $\pm$ .002\end{tabular}                                                                                                                                                                                       & \begin{tabular}[c]{@{}c@{}}0.085 $\pm$ .008\\ 0.086 $\pm$ .011\\ 0.085 $\pm$ .016\\ 0.084 $\pm$ .017\\ 0.081 $\pm$ .011\end{tabular}                                                                                                                                                                              & \begin{tabular}[c]{@{}c@{}}0.91 $\pm$ .00\\ 0.95 $\pm$ .02\\ 0.94 $\pm$ .01\\ 0.93 $\pm$ .02\\ 0.91 $\pm$ .01\end{tabular}                                                                                                                                                            \\ \hline
\begin{tabular}[c]{@{}c@{}}CAE\end{tabular}                 & \begin{tabular}[c]{@{}c@{}}4\\ 4\\ 4\\ 8\\ 8\\ 8\\ 16\\ 16\\ 16\\ 32\\ 32\\ 32\\ 32\\ 32\end{tabular} & \begin{tabular}[c]{@{}c@{}}4\\ 8\\ 16\\ 4\\ 8\\ 16\\ 4\\ 8\\ 16\\ 4\\ 8\\ 16\\ 32\\ 64\end{tabular} & \begin{tabular}[c]{@{}c@{}}601,452\\ 612,700\\ 635,196\\ 635,196\\ 875,568\\ 917,328\\ 1,361,160\\ 1,401,304\\ 1,481,592\\ 2,374,104\\ 2,452,776\\ 2,610,120\\ 2,924,808\\ 3,554,184\end{tabular} & \begin{tabular}[c]{@{}c@{}}0.0516 $\pm$ .001\\ 0.0452 $\pm$ .002\\ 0.0409 $\pm$ .001 \\ 0.0414 $\pm$ .001\\ 0.0420 $\pm$ .002\\ 0.0374 $\pm$  .001\\ 0.0395 $\pm$ .002\\ 0.0389 $\pm$ .001\\ 0.0317 $\pm$ .001\\ 0.0286 $\pm$ .002\\ 0.0225 $\pm$ .002\\ 0.0290 $\pm$ .001\\ 0.0289 $\pm$ .001\\ 0.0282 $\pm$ .002\end{tabular} & \begin{tabular}[c]{@{}c@{}}0.069 $\pm$ .019\\ 0.068 $\pm$ .024\\ 0.065 $\pm$ .018 \\ 0.061 $\pm$ .017\\ 0.052 $\pm$ .011\\ 0.047 $\pm$  .015\\ 0.079 $\pm$ .013\\ 0.068 $\pm$ .017\\ 0.051 $\pm$ .018\\ 0.058 $\pm$ .016\\ 0.052 $\pm$ .015\\ 0.043 $\pm$ .012\\ 0.045 $\pm$ .011\\ 0.038 $\pm$ .017\end{tabular} & \begin{tabular}[c]{@{}c@{}}0.92 $\pm$ .01\\ 0.93 $\pm$ .02\\ 0.94 $\pm$ .01 \\ 0.94 $\pm$ .01\\ 0.98 $\pm$ .02\\ 0.92 $\pm$  .01\\ 0.93 $\pm$ .02\\ 0.94 $\pm$ .01\\ 0.94 $\pm$ .01\\ 0.98 $\pm$ .02\\ 0.98 $\pm$ .02\\ 0.98 $\pm$ .01\\ 0.98 $\pm$ .01\\ 0.98 $\pm$ .02\end{tabular} \\ \hline
\end{tabular}
}
\end{table}


\begin{table}[ht]
\caption{Extensive Comparison on FMNIST.}
\label{tab:FMNIST}
\resizebox{\textwidth}{!}{
\begin{tabular}{|c|c|c|c|c|c|c|}
\hline
\textbf{\begin{tabular}[c]{@{}c@{}}Model\\ \end{tabular}}     & \textbf{\# of charts}                                                                                & \textbf{\begin{tabular}[c]{@{}c@{}}Latent \\ dim\end{tabular}}                                  & \textbf{\begin{tabular}[c]{@{}c@{}}\# of \\  Param.\end{tabular}}                                                                                                                                 & \textbf{\begin{tabular}[c]{@{}c@{}}Recon.\\  Error\end{tabular}}                                                                                                                                                                                                                                                               & \textbf{Unfaithfulness}                                                                                                                                                                                                                                                                                          & \textbf{Coverage}                                                                                                                                                                                                                                                                 \\ \hline
\begin{tabular}[c]{@{}c@{}}Small\\ VAE \\ \end{tabular} & \begin{tabular}[c]{@{}c@{}}1\\ 1\\ 1\\ 1\\ 1\end{tabular}                                             & \begin{tabular}[c]{@{}c@{}}4\\ 8\\ 16\\ 32\\ 64\end{tabular}                                        & \begin{tabular}[c]{@{}c@{}}893,028\\ 896,032\\ 902,302\\ 914,056\\ 938,088\end{tabular}                                                                                                           & \begin{tabular}[c]{@{}c@{}}0.0575 $\pm$ .000\\ 0.0502 $\pm$ .001\\ 0.0477 $\pm$ .003\\ 0.0482 $\pm$ .001\\ 0.0468 $\pm$ .001\end{tabular}                                                                                                                                                                                      & \begin{tabular}[c]{@{}c@{}}0.016 $\pm$ .021\\ 0.015 $\pm$ .032\\ 0.017 $\pm$ .021\\ 0.019 $\pm$ .041\\ 0.029 $\pm$ .034\end{tabular}                                                                                                                                                                             & \begin{tabular}[c]{@{}c@{}}0.80 $\pm$ .01\\ 0.84 $\pm$ .01\\ 0.88 $\pm$ .01\\ 0.92 $\pm$ .02\\ 0.95 $\pm$ .01\end{tabular}                                                                                                                                                        \\ \hline
\begin{tabular}[c]{@{}c@{}}Large\\ VAE\end{tabular}                         & \begin{tabular}[c]{@{}c@{}}1\\ 1\\ 1\\ 1\\ 1\end{tabular}                                             & \begin{tabular}[c]{@{}c@{}}4\\ 8\\ 16\\ 32\\ 64\end{tabular}                                        & \begin{tabular}[c]{@{}c@{}}2,535,028\\ 2,541,032\\ 2,553,040\\ 2,577,056\\ 2,625,088\end{tabular}                                                                                                 & \begin{tabular}[c]{@{}c@{}}0.0474 $\pm$ .001\\ 0.0405 $\pm$ .000\\ 0.0389 $\pm$ .000\\ 0.0309 $\pm$ .000\\ 0.0291 $\pm$ .007\end{tabular}                                                                                                                                                                                      & \begin{tabular}[c]{@{}c@{}}0.014 $\pm$ .008\\ 0.013 $\pm$ .011\\ 0.015 $\pm$ .016\\ 0.016 $\pm$ .017\\ 0.021 $\pm$ .011\end{tabular}                                                                                                                                                                             & \begin{tabular}[c]{@{}c@{}}0.92 $\pm$ .00\\ 0.93 $\pm$ .02\\ 0.94 $\pm$ .01\\ 0.94 $\pm$ .02\\ 0.92 $\pm$ .01\end{tabular}                                                                                                                                                        \\ \hline
CAE                                                               & \begin{tabular}[c]{@{}c@{}}4\\ 4\\ 4\\ 8\\ 8\\ 8\\ 16\\ 16\\ 16\\ 32\\ 32\\ 32\\ 32\\ 32\end{tabular} & \begin{tabular}[c]{@{}c@{}}4\\ 8\\ 16\\ 4\\ 8\\ 16\\ 4\\ 8\\ 16\\ 4\\ 8\\ 16\\ 32\\ 64\end{tabular} & \begin{tabular}[c]{@{}c@{}}601,452\\ 612,700\\ 635,196\\ 635,196\\ 875,568\\ 917,328\\ 1,361,160\\ 1,401,304\\ 1,481,592\\ 2,374,104\\ 2,452,776\\ 2,610,120\\ 2,924,808\\ 3,554,184\end{tabular} & \begin{tabular}[c]{@{}c@{}}0.0409 $\pm$ .001\\ 0.0351 $\pm$ .002\\ 0.0301 $\pm$ .001 \\ 0.0314 $\pm$ .001\\ 0.0320 $\pm$ .002\\ 0.0274 $\pm$  .001\\ 0.0395 $\pm$ .002\\ 0.0389 $\pm$ .001\\ 0.0314 $\pm$ .001\\ 0.0281 $\pm$ .002\\ 0.0225 $\pm$ .002\\ 0.0190 $\pm$ .001\\ 0.0189 $\pm$ .001\\ 0.0177 $\pm$ .002\end{tabular} & \begin{tabular}[c]{@{}c@{}}0.010 $\pm$ .001\\ 0.015 $\pm$ .002\\ 0.010 $\pm$ .001 \\ 0.011 $\pm$ .001\\ 0.012 $\pm$ .002\\ 0.017 $\pm$  .001\\ 0.019 $\pm$ .002\\ 0.018 $\pm$ .001\\ 0.011 $\pm$ .001\\ 0.018 $\pm$ .002\\ 0.012 $\pm$ .002\\ 0.016 $\pm$ .001\\ 0.015 $\pm$ .001\\ 0.007 $\pm$ .002\end{tabular} & \begin{tabular}[c]{@{}c@{}}0.9 $\pm$ .01\\ 0.9 $\pm$ .01\\ 0.9 $\pm$ .01 \\ 0.94 $\pm$ .01\\ 0.93 $\pm$ .02\\ 0.97 $\pm$  .02\\ 0.98 $\pm$ .02\\ 0.96 $\pm$ .02\\ 0.97 $\pm$ .01\\ 0.94 $\pm$ .01\\ 0.96 $\pm$ .02\\ 0.97 $\pm$ .02\\ 0.97 $\pm$ .01\\ 0.97 $\pm$ .01\end{tabular} \\ \hline
\end{tabular}
}
\end{table}


\section{Network architectures}\label{app:Networks}
We section provides the details of the neural network architectures used in the numerical experiments. We denote by $FC_m$ a fully connected layer with $m$ output neurons; by $Conv_{i,j,k.l}$ a convolution layer with filters of size $(i,j)$, input dimension $k$, and output dimension $l$; by $d$ the dimension of the latent space, $m$ the dimension of the ambient space and $N$ the number of charts. See~\eqref{eqn:model.ae},\eqref{eqn:model.small.vae},\eqref{eqn:model.medium.vae},\eqref{eqn:model.large.vae},\eqref{eqn:model.cae1} \eqref{eqn:model.cae2}, and \eqref{eqn:model.cae3} for the architectures.

\begin{figure*}[h]
\textbf{Auto-Encoders}:
\begin{equation}\label{eqn:model.ae}
  \begin{split}
    \text{Encoder}&: x \rightarrow FC_{250} \rightarrow  FC_{250} \rightarrow  FC_{250} \rightarrow FC_{d} \rightarrow  z \\
    \text{Decoder}&: z  \rightarrow FC_{250} \rightarrow  FC_{250} \rightarrow  FC_{250} \rightarrow FC_{m} \rightarrow  y
  \end{split}
\end{equation}

\textbf{Cat Variational Auto-Encoder}:
\begin{equation}\label{eqn:model.small.vae}
  \begin{split}
    \text{Encoder}&: x \rightarrow FC_{50} \rightarrow  FC_{50} \rightarrow FC_{2d} \rightarrow  \mu, \sigma \\
    \text{Decoder}&: z \in \mathcal{N}(\mu,\sigma) \rightarrow FC_{50} \rightarrow  FC_{50} \rightarrow FC_{m} \rightarrow  y
  \end{split}
\end{equation}

\textbf{Small Variational Auto-Encoders}:
\begin{equation}\label{eqn:model.medium.vae}
  \begin{split}
    \text{Encoder}&: x \rightarrow FC_{100} \rightarrow  FC_{100} \rightarrow  FC_{100} \rightarrow FC_{2d} \rightarrow  \mu, \sigma \\
    \text{Decoder}&: z \in \mathcal{N}(\mu,\sigma) \rightarrow FC_{100} \rightarrow  FC_{100} \rightarrow  FC_{100} \rightarrow FC_{m} \rightarrow  y
  \end{split}
\end{equation}

\textbf{Large Variational Auto-Encoders}:
\begin{equation}\label{eqn:model.large.vae}
  \begin{split}
    \text{Encoder}&: x \rightarrow FC_{250} \rightarrow  FC_{250} \rightarrow  FC_{250} \rightarrow FC_{2d} \rightarrow  \mu, \sigma \\
    \text{Decoder}&: z \in \mathcal{N}(\mu,\sigma) \rightarrow FC_{250} \rightarrow  FC_{250} \rightarrow  FC_{250} \rightarrow FC_{m} \rightarrow  y
  \end{split}
\end{equation}

\textbf{Small CAE}
\begin{equation}\label{eqn:model.cae1}
  \begin{split}
    \text{Initial Encoder}&: x \rightarrow FC_{100} \rightarrow  FC_{100} \rightarrow  FC_{2d} \rightarrow z \\
    \text{Chart Encoders}  &: z \rightarrow FC_{100} \rightarrow FC_{100} \rightarrow FC_{d} \rightarrow z_{\alpha} \\
    \text{Chart Decoders}  &: z_{\alpha} \rightarrow  FC_{100} \rightarrow  FC_{100} \rightarrow FC_{2d} \rightarrow       \bar{z_{\alpha}} \\
     \text{Final Decoder} &: \bar{z_{\alpha}} \rightarrow FC_{100} \rightarrow FC_{100} \rightarrow FC_{m} \rightarrow y_{\alpha} \\
    \text{Chart Prediction}&: x \rightarrow FC_{100} \rightarrow FC_{100}  \rightarrow  FC_{N} \rightarrow softmax \rightarrow p
  \end{split}
\end{equation}

\textbf{Large CAE}
\begin{equation}\label{eqn:model.cae2}
  \begin{split}
    \text{Initial Encoder}&: x \rightarrow FC_{100} \rightarrow  FC_{100} \rightarrow FC_{100} \rightarrow  FC_{100} \rightarrow  FC_{2d} \rightarrow z \\
    \text {Chart Encoders} &:z \rightarrow FC_{100} \rightarrow FC_{100} \rightarrow  FC_{100}\rightarrow FC_{100} \rightarrow FC_{d} \rightarrow z_{\alpha} \\
    \text{Chart Decoders}&:z_{\alpha} \rightarrow  FC_{100} \rightarrow FC_{100} \rightarrow  FC_{100}\rightarrow  FC_{100} \rightarrow FC_{2d} \rightarrow  \bar{z_{\alpha}} \ \\
     \text{Final Decoder} &: \bar{z_{\alpha}} \rightarrow FC_{100}\rightarrow FC_{100} \rightarrow  FC_{100} \rightarrow FC_{100} \rightarrow FC_{m} \rightarrow y_{\alpha} \\
    \text{Chart Prediction}&: x \rightarrow FC_{100} \rightarrow FC_{100} \rightarrow FC_{100} \rightarrow  FC_{100} \rightarrow  FC_{N} \rightarrow softmax \rightarrow p
  \end{split}
\end{equation}

\textbf{Conv CAE}
\begin{equation}\label{eqn:model.cae3}
  \begin{split}
    \text{Initial Encoder}&: x \rightarrow FC_{150} \rightarrow  FC_{150} \rightarrow  FC_{625} \rightarrow z  \\
    \text {Chart Encoder} &:z\rightarrow Conv_{3,3,1,8} \rightarrow  Conv_{3,3,8,8} \rightarrow  Conv_{3,3,8,16} \rightarrow z \\
    \text{Chart Decoder}&:z_{\alpha}\rightarrow Conv_{3,3,16,8} \rightarrow  Conv_{3,3,8,8} \rightarrow  Conv_{3,3,8,1} \rightarrow FC_{m} \rightarrow  y_{\alpha} \\
    \text{Chart Prediction}&: z \rightarrow FC_{250} \rightarrow  FC_{10} \rightarrow softmax \rightarrow p
  \end{split}
\end{equation}

\end{figure*}

\end{document}